\newcommand*{\COLT}{}

\newcommand*{\CAMREADY}{}

\ifdefined\NIPS
	\documentclass{article}
	\usepackage{nips15submit_e,times}
	\usepackage{hyperref}
	\usepackage{url}
\fi
\ifdefined\CVPR
	\documentclass[10pt,twocolumn,letterpaper]{article}
	\usepackage{cvpr}
	\usepackage{times}
	\usepackage{epsfig}
	\usepackage{graphicx}
	\usepackage{amsmath}
	\usepackage{amssymb}
\fi
\ifdefined\AISTATS
	\documentclass[twoside]{article}
	\ifdefined\CAMREADY
		\usepackage[accepted]{aistats2015}
	\else
		\usepackage{aistats2015}
	\fi
	\usepackage{url}
\fi
\ifdefined\COLT
	\ifdefined\CAMREADY
		\documentclass{colt2016}
	\else
		\documentclass[anon,12pt]{colt2016}
	\fi
	\usepackage{times}
\fi

\usepackage{color}
\usepackage{amsfonts}
\usepackage{amsmath}
\usepackage{algorithm}
\usepackage{algorithmic}
\usepackage{graphicx}
\usepackage{nicefrac}
\usepackage{bbm}
\usepackage{wrapfig}
\usepackage{tikz}
\usepackage[titletoc,title]{appendix}
\ifdefined\NIPS
	\usepackage[square,comma,numbers]{natbib}
\fi
\ifdefined\CVPR
	\usepackage[square,comma,numbers]{natbib}
\fi
\ifdefined\AISTATS
	\usepackage[round,authoryear]{natbib}
\fi
\ifdefined\COLT
\else
	\usepackage{amsthm}
\fi

\ifdefined\COLT
\else
	
	\newtheorem{lemma}{Lemma}
	\newtheorem{corollary}{Corollary}
	\newtheorem{theorem}{Theorem}

\fi

\def\be{\begin{equation}}
\def\ee{\end{equation}}
\def\beas{\begin{eqnarray*}}
\def\eeas{\end{eqnarray*}}
\def\bea{\begin{eqnarray}}
\def\eea{\end{eqnarray}}

\newcommand{\x}{{\mathbf x}}
\newcommand{\y}{{\mathbf y}}

\newcommand{\uu}{{\mathbf u}}
\newcommand{\vv}{{\mathbf v}}
\newcommand{\w}{{\mathbf w}}

\newcommand{\e}{{\mathbf e}}
\newcommand{\aaa}{{\mathbf a}}
\newcommand{\bb}{{\mathbf b}}

\newcommand{\1}{{\mathbf 1}}
\newcommand{\0}{{\mathbf 0}}
\newcommand{\A}{{\mathcal A}}
\newcommand{\B}{{\mathcal B}}
\newcommand{\D}{{\mathcal D}}
\newcommand{\E}{{\mathcal E}}
\newcommand{\F}{{\mathcal F}}
\newcommand{\G}{{\mathcal G}}
\newcommand{\HH}{{\mathcal H}}

\newcommand{\M}{{\mathcal M}}

\newcommand{\T}{{\mathcal T}}
\newcommand{\X}{{\mathcal X}}
\newcommand{\Y}{{\mathcal Y}}

\newcommand{\EE}{{\mathbb E}}
\newcommand{\R}{{\mathbb R}}
\newcommand{\N}{{\mathbb N}}

\newcommand{\mubf}{{\boldsymbol{\mu}}}
\newcommand{\sigmabf}{{\boldsymbol{\sigma}}}

\newcommand{\abs}[1]{\left\lvert#1 \right\rvert}

\newcommand{\indc}[1]{\mathbbm{1}\left[#1\right]}

\DeclareMathOperator*{\argmax}{argmax} 
\DeclareMathOperator*{\argmin}{argmin}

\ifdefined\CVPR
	\usepackage[breaklinks=true,bookmarks=false]{hyperref}
	\ifdefined\CAMREADY
		\cvprfinalcopy 
	\fi
	
\fi
\ifdefined\NIPS

	\ifdefined\CAMREADY
		\nipsfinalcopy
	\fi
\fi

\begin{document}

\ifdefined\NIPS
	\title{On the Expressive Power of Deep Learning: A Tensor Analysis}
	\author{
	Nadav Cohen \\
	The Hebrew University of Jerusalem \\
	\texttt{cohennadav@cs.huji.ac.il} \\
	\And 
	Or Sharir \\
	The Hebrew University of Jerusalem \\
	\texttt{or.sharir@cs.huji.ac.il} \\
	\And 
	Amnon Shashua \\
	The Hebrew University of Jerusalem \\
	Mobileye N.V.\\
	OrCam Ltd.\\
	\texttt{shashua@cs.huji.ac.il} \\
	}
	\maketitle
\fi
\ifdefined\CVPR
	\title{On the Expressive Power of Deep Learning: A Tensor Analysis}
	\author{
	Nadav Cohen \\
	The Hebrew University of Jerusalem \\	
	\texttt{cohennadav@cs.huji.ac.il} \\
	\and
	Or Sharir \\
	The Hebrew University of Jerusalem \\
	\texttt{or.sharir@cs.huji.ac.il} \\
	\and    
	Amnon Shashua \\
	The Hebrew University of Jerusalem \\
	Mobileye N.V.\\
	OrCam Ltd.\\
	\texttt{shashua@cs.huji.ac.il} \\
	}
	\maketitle
\fi
\ifdefined\AISTATS
	\twocolumn[
	\aistatstitle{On the Expressive Power of Deep Learning: A Tensor Analysis}
	\ifdefined\CAMREADY
		\aistatsauthor{Nadav Cohen \And Or Sharir \And Amnon Shashua}
		\aistatsaddress{The Hebrew University of Jerusalem \And The Hebrew University of Jerusalem \And The Hebrew University of Jerusalem}
	\else
		\aistatsauthor{Anonymous Author 1 \And Anonymous Author 2 \And Anonymous Author 3}
		\aistatsaddress{Unknown Institution 1 \And Unknown Institution 2 \And Unknown Institution 3}
	\fi
	]	
\fi
\ifdefined\COLT
	\title{On the Expressive Power of Deep Learning: A Tensor Analysis}
	\coltauthor{\Name{Nadav Cohen} \Email{cohennadav@cs.huji.ac.il}\\
	\Name{Or Sharir} 				 \Email{or.sharir@cs.huji.ac.il}\\
	\Name{Amnon Shashua} 		 \Email{shashua@cs.huji.ac.il}\\
	\addr The Hebrew University of Jerusalem}
	\maketitle
\fi

\begin{abstract}

It has long been conjectured that hypotheses spaces suitable for data that is compositional in nature, such as text or images, may be more efficiently represented with deep hierarchical networks than with shallow ones. 
Despite the vast empirical evidence supporting this belief, theoretical justifications to date are limited.
In particular, they do not account for the locality, sharing and pooling constructs of convolutional networks, the most successful deep learning architecture to date.
In this work we derive a deep network architecture based on arithmetic circuits that inherently employs locality, sharing and pooling.
An equivalence between the networks and hierarchical tensor factorizations is established.   
We show that a shallow network corresponds to CP (rank-1) decomposition, whereas a deep network corresponds to Hierarchical Tucker decomposition.  
Using tools from measure theory and matrix algebra, we prove that besides a negligible set, all functions that can be implemented by a deep 
network of polynomial size, require exponential size in order to be realized (or even approximated) by a shallow network. 
Since log-space computation transforms our networks into SimNets, the result applies directly to a deep learning architecture demonstrating promising empirical performance.
The construction and theory developed in this paper shed new light on various practices and ideas employed by the deep learning community.
\ifdefined\CAMREADY
\else
	\footnote{The authors wish this manuscript to be considered for a student paper award.}
\fi

\end{abstract}

\ifdefined\COLT
	\medskip
\fi

\section{Introduction} \label{sec:intro}

The expressive power of neural networks is achieved through depth. 
There is mounting empirical evidence that for a given budget of resources (e.g.~neurons), the deeper one goes, the better the eventual performance will be.
However, existing theoretical arguments that support this empirical finding are limited.
There have been many attempts to theoretically analyze function spaces generated by network architectures, and their dependency on
network depth and size.
The prominent approach for justifying the power of depth is to show that deep networks can efficiently express functions that would require shallow networks to have super-polynomial size.
We refer to such scenarios as instances of \emph{depth efficiency}.
Unfortunately, existing results dealing with depth efficiency (e.g.~\cite{hastad1986almost,Hastad91,bengio2011shallow,martens2014expressive}) typically apply to specific network architectures that do not resemble ones commonly used in practice.
In particular, none of these results apply to convolutional networks (\cite{lecun1995convolutional}), which represent the most empirically successful and widely used deep learning architecture to date.
A further limitation of current results is that they merely show \emph{existence} of depth efficiency (i.e.~of functions that are efficiently realizable with a certain depth but cannot be efficiently realized with shallower depths), without providing any information as to how frequent this property is.  
These shortcomings of current theory are the ones that motivated our work.

The architectural features that specialize convolutional networks compared to classic feed-forward fully-connected networks are threefold.
The first feature, \emph{locality}, refers to the connection of a neuron only to neighboring neurons in the preceding layer, as opposed to having the entire layer drive it.
In the context of image processing (the most common application of convolutional networks), locality is believed to reflect the inherent compositional structure of data~--~the closer pixels are in an image, the more likely they are to be correlated.
The second architectural feature of convolutional networks is \emph{sharing}, which means that different neurons in the same layer, connected to different neighborhoods in the preceding layer, share the same weights.
Sharing, which together with locality gives rise to convolution, is motivated by the fact that in natural images, the semantic meaning of a pattern often does not depend on its location (i.e.~two identical patterns appearing in different locations of an image often convey the same semantic content).
Finally, the third architectural idea of convolutional networks is \emph{pooling}, which is essentially an operator that decimates layers, replacing neural activations in a spatial window by a single value (e.g.~their maximum or average).
In the context of images, pooling induces invariance to translations (which often do not affect semantic content), and in addition is believed to create a hierarchy of abstraction in the patterns neurons respond to.
The three architectural elements of locality, sharing and pooling, which have facilitated the great success of convolutional networks, are all lacking in existing theoretical studies of depth efficiency.

In this paper we introduce a \emph{convolutional arithmetic circuit} architecture that incorporates locality, sharing and pooling.
Arithmetic circuits (also known as Sum-Product Networks,~\cite{Poon-Domingos2011}) are networks with two types of nodes: sum nodes, which compute a weighted sum of their inputs, and product nodes, computing the product of their inputs. 
We use sum nodes to implement convolutions (locality with sharing), and product nodes to realize pooling.
The models we arrive at may be viewed as convolutional networks with product pooling and linear point-wise activation.
They are attractive on three accounts.
First, as discussed in app.~\ref{app:simnets}, convolutional arithmetic circuits are equivalent to SimNets, a new deep learning architecture that has recently demonstrated promising empirical results on various image recognition benchmarks~(\cite{simnets2}).
Second, as we show in sec.~\ref{sec:cac}, convolutional arithmetic circuits are realizations of hierarchical tensor decompositions (see~\cite{Hackbusch-book}), opening the door to various mathematical and algorithmic tools for their analysis and implementation.
Third, the depth efficiency of convolutional arithmetic circuits, which we analyze in sec.~\ref{sec:theorems}, was shown in the subsequent work of~\cite{cohen2016convolutional} to be superior to the depth efficiency of the popular convolutional rectifier networks, namely convolutional networks with rectified linear (ReLU) activation and max or average pooling.

Employing machinery from measure theory and matrix algebra, made available through their connection to hierarchical tensor decompositions, we prove a number of fundamental results concerning the depth efficiency of our convolutional arithmetic circuits.
Our main theoretical result (thm.~\ref{thm:fundamental} and corollary~\ref{corollary:fundamental}) states that \emph{besides a negligible (zero measure) set, all functions that can be realized by a deep network of polynomial size, require exponential size in order to be realized, or even approximated, by a shallow network}.
When translated to the viewpoint of tensor decompositions, this implies that \emph{almost all} tensors realized by Hierarchical Tucker (HT) decomposition (\cite{Hackbusch:2009jj}) cannot be efficiently realized by the classic CP (rank-1) decomposition.
To the best of our knowledge, this result is unknown to the tensor analysis community, in which the advantage of HT over CP is typically demonstrated through \emph{specific examples} of tensors that can be efficiently realized by the former and not by the latter.
Following our main result, we present a generalization (thm.~\ref{thm:generalized} and corollary~\ref{corollary:generalized}) that compares networks of arbitrary depths, showing that the amount of resources one has to pay in order to maintain representational power while trimming down layers of a network grows double exponentially w.r.t. the number of layers cut off.
We also characterize cases in which dropping a single layer bears an exponential price.

The remainder of the paper is organized as follows.
In sec.~\ref{sec:preliminaries} we briefly review notations and mathematical background required in order to follow our work.
This is followed by sec.~\ref{sec:cac}, which presents our convolutional arithmetic circuits and establishes their equivalence with tensor decompositions.
Our theoretical analysis is covered in sec.~\ref{sec:theorems}.
Finally, sec.~\ref{sec:discussion} concludes.
In order to keep the manuscript at a reasonable length, we defer our detailed survey of related work to app.~\ref{app:related_work}, covering works on the depth efficiency of boolean circuits, arithmetic circuits and neural networks, as well as different applications of tensor analysis in the field of deep learning.

\section{Preliminaries} \label{sec:preliminaries}

We begin by establishing notational conventions that will be used throughout the paper.  
We denote vectors using bold typeface, e.g. $\vv \in \R^s$.  
The coordinates of such a vector are referenced with regular typeface and a subscript, e.g. $v_i \in \R$.  
This is not to be confused with \emph{bold} typeface and a subscript, e.g. $\vv_i \in \R^s$, which represents a vector that belongs to some sequence.  
Tensors (multi-dimensional arrays) are denoted by the letters ``A'' and ``B'' in calligraphic typeface, e.g. $\A, \B \in \R^{M_1 \times \cdots \times M_N}$.  
A specific entry in a tensor will be referenced with subscripts, e.g. $\A_{d_1{\ldots}d_N} \in \R$.  
Superscripts will be used to denote individual objects within a collection.  
For example, $\vv^{(i)}$ stands for vector $i$ and $\A^y$ stands for tensor $y$.  
In cases where the collection of interest is indexed by multiple coordinates, we will have multiple superscripts referencing individual objects, e.g. $\aaa^{l,j,\gamma}$ will stand for vector $(l,j,\gamma)$.  
As shorthand for the Cartesian product of the Euclidean space $\R^s$ with itself $N$ times, we will use the notation $(\R^s)^N$.  
Finally, for a positive integer $k$ we use the shorthand $[k]$ to denote the set $\{1,\ldots,k\}$.

We now turn to establish a baseline, i.e.~to present basic definitions and results, in the broad and comprehensive field of tensor analysis.  
We list here only the essentials required in order to follow the paper, referring the interested reader to~\cite{Hackbusch-book} for a more complete introduction to the field
\footnote{
The definitions we give are concrete special cases of the more abstract algebraic definitions given in~\cite{Hackbusch-book}.
We limit the discussion to these special cases since they suffice for our needs and are easier to grasp.
}.  
The most straightforward way to view a tensor is simply as a multi-dimensional array: $\A_{d_1,\ldots,d_N}\in\R$ where $i\in[N],d_i\in[M_i]$.
The number of indexing entries in the array, which are also called \emph{modes}, is referred to as the \emph{order} of the tensor.  
The term \emph{dimension} stands for the number of values an index can take in a particular mode.  
For example, the tensor $\A$ appearing above has order $N$ and dimension $M_i$ in mode $i$, $i\in[N]$.  
The space of all possible configurations $\A$ can take is called a \emph{tensor space} and is denoted, quite naturally, by $\R^{M_1 \times \cdots \times M_N}$.

A central operator in tensor analysis is the \emph{tensor product}, denoted $\otimes$.
This operator intakes two tensors $\A$ and $\B$ of orders $P$ and $Q$ respectively, and returns a tensor $\A \otimes \B$ of order $P+Q$, defined by: $\left(\A \otimes \B\right)_{d_1{\ldots}d_{P+Q}}=\A_{d_1{\ldots}d_P} \cdot \B_{d_{P+1}{\ldots}d_{P+Q}}$.
Notice that in the case $P=Q=1$, the tensor product reduces to an outer product between vectors.
Specifically, $\vv \otimes \uu$~--~the tensor product between $\uu \in \R^{M_1}$ and $\vv \in \R^{M_2}$, is no other than the rank-1 matrix $\vv \uu^\top \in \R^{M_1 \times M_2}$.  
In this context, we will often use the shorthand $\mathop{\otimes}_{i=1}^N \vv^{(i)}$ to denote the joint tensor product $\vv^{(1)} \otimes \cdots \otimes \vv^{(N)}$.  

Tensors of the form $\mathop{\otimes}_{i=1}^N \vv^{(i)}$ are called \emph{pure} or \emph{elementary}, and are regarded as having \emph{rank-1} (assuming $\vv^{(i)} \neq 0~~\forall i$).  
It is not difficult to see that any tensor can be expressed as a sum of rank-1 tensors:
\be
\A = \sum_{z=1}^Z \vv^{(1)}_z \otimes \cdots \otimes \vv^{(N)}_z \quad ,\vv_z^{(i)} \in \R^{M_i}  
\label{eq:cp_decomp_def}
\ee
A representation as above is called a CANDECOMP/PARAFAC decomposition of $\A$, or in short, a \emph{CP decomposition}
\footnote{ 
CP decomposition is regarded as the classic and most basic tensor decomposition, dating back to the beginning of the 20'th century (see~\cite{Kolda-Bader2009} for a historic survey).
}.
The \emph{CP-rank} of $\A$ is defined as the minimum number of terms in a CP decomposition, i.e. as the minimal $Z$ for which eq.~\ref{eq:cp_decomp_def} can hold.  
Notice that for a tensor of order 2, i.e. a matrix, this definition of CP-rank coincides with that of standard matrix rank.

A \emph{symmetric tensor} is one that is invariant to permutations of its indices.  
Formally, a tensor $\A$ of order $N$ which is symmetric will have equal dimension $M$ in all modes, and for every permutation $\pi:[N]\to[N]$ and indices $d_1{\ldots}d_N \in [M]$, the following equality will hold: $\A_{d_{\pi(1)}{\ldots}d_{\pi(N)}} = \A_{d_1{\ldots}d_N}$.  
Note that for a vector $\vv\in\R^M$, the tensor $\mathop{\otimes}_{i=1}^N \vv \in \R^{M\times \cdots \times M}$ is symmetric.  
Moreover, every symmetric tensor may be expressed as a linear combination of such (symmetric rank-1) tensors: $\A = \sum_{z=1}^Z \lambda_z \cdot \vv_z \otimes \cdots \otimes \vv_z$.
This is referred to as a \emph{symmetric CP decomposition}, and the \emph{symmetric CP-rank} is the minimal $Z$ for which such a decomposition exists.  
Since a symmetric CP decomposition is in particular a standard CP decomposition, the symmetric CP-rank of a symmetric tensor is always greater or equal to its standard CP-rank.
Note that for the case of symmetric matrices (order-$2$ tensors) the symmetric CP-rank and the original CP-rank are always equal.  

A repeating concept in this paper is that of \emph{measure zero}.  
More broadly, our analysis is framed in measure theoretical terms.  
While an introduction to the field is beyond the scope of the paper (the interested reader is referred to~\cite{jones2001lebesgue}), it is possible to intuitively grasp the ideas that form the basis to our claims.  
When dealing with subsets of a Euclidean space, the standard and most natural measure in a sense is called the \emph{Lebesgue measure}.  
This is the only measure we consider in our analysis.
A set of (Lebesgue) measure zero can be thought of as having zero ``volume'' in the space of interest.  
For example, the interval between $(0,0)$ and $(1,0)$ has zero measure as a subset of the 2D plane, but has positive measure as a subset of the 1D $x$-axis.  
An alternative way to view a zero measure set $S$ follows the property that if one draws a random point in space by some continuous distribution, the probability of that point hitting $S$ is necessarily zero.  
A related term that will be used throughout the paper is \emph{almost everywhere}, which refers to an entire space excluding, at most, a set of zero measure.

\section{Convolutional Arithmetic Circuits} \label{sec:cac}

We consider the task of classifying an instance $X=(\x_1,\ldots,\x_N)$, $\x_i \in \R^s$, into one of the categories $\Y:=\{1,\ldots,Y\}$.  
Representing instances as collections of vectors is natural in many applications.
In the case of image processing for example, $X$ may correspond to an image, and $\x_1\ldots\x_N$ may correspond to vector arrangements of (possibly overlapping) patches around pixels.  
As customary, classification is carried out through maximization of per-label score functions $\{h_y\}_{y\in\Y}$, i.e.~the predicted label for the instance $X$ will be the index $y\in\Y$ for which the score value $h_y(X)$ is maximal.  
Our attention is thus directed to functions over the instance space $\X:=\{(\x_1,\ldots,\x_N):\x_i \in \R^s\}=(\R^s)^N$.  
We define our hypotheses space through the following representation of score functions:
\be
h_{y}\left(\x_1,\ldots,\x_N\right)=\sum_{d_1{\ldots}d_N=1}^M\A_{d_1,\ldots,d_N}^{y}\prod_{i=1}^{N} f_{\theta_{d_i}}(\x_i)
\label{eq:score}
\ee
$f_{\theta_1}{\ldots}f_{\theta_M}:\R^s\to\R$ are referred to as \emph{representation functions}, selected from a parametric family $\F=\left\{ f_{\theta}:\R^s\to\R\right\} _{\theta\in\Theta}$.
Natural choices for this family are wavelets, radial basis functions (\emph{Gaussians}), and affine functions followed by point-wise activation (\emph{neurons}).
The \emph{coefficient tensor} $\A^y$ has order $N$ and dimension $M$ in each mode.
Its entries correspond to a basis of $M^N$ point-wise product functions $\{(\x_1,\ldots,\x_N)\mapsto\prod_{i=1}^N f_{\theta_{d_i}}(\x_i) \}_{d_1{\ldots}d_N\in[M]}$.
We will often consider fixed linearly independent representation functions $f_{\theta_1}{\ldots}f_{\theta_M}$.
In this case the point-wise product functions are linearly independent as well (see app.~\ref{app:hypo_space:preliminaries}), and we have a one to one correspondence between score functions and coefficient tensors.
To keep the manuscript concise, we defer the derivation of our hypotheses space (eq.~\ref{eq:score}) to app.~\ref{app:hypo_space}, noting here that it arises naturally from the notion of tensor products between $L^2$~spaces.

Our eventual aim is to realize score functions $h_y$ with a layered network architecture.  
As a first step along this path, we notice that $h_y(\x_1,\ldots,\x_N)$ is fully determined by the activations of the $M$ representation functions $f_{\theta_1}{\ldots}f_{\theta_M}$ on the $N$ input vectors $\x_1{\ldots}\x_N$.
In other words, given $\{f_{\theta_d}(\x_i)\}_{d\in[M],i\in[N]}$, the score $h_y(\x_1,\ldots,\x_N)$ is independent of the input.
It is thus natural to consider the computation of these $M{\cdot}N$ numbers as the first layer of our networks.
This layer, referred to as the \emph{representation layer}, may be conceived as a convolutional operator with $M$ channels, each corresponding to a different function applied to all input vectors (see fig.~\ref{fig:cp_model}).

Once we have constrained our score functions to have the structure depicted in eq.~\ref{eq:score}, learning a classifier reduces to estimation of the parameters $\theta_{1}{\ldots}\theta_{M}$, and the coefficient tensors $\A^1{\ldots}\A^Y$.  
The computational challenge is that the latter tensors are of order $N$ (and dimension $M$ in each mode), having an exponential number of entries ($M^N$ each).  
In the next subsections we utilize tensor decompositions (factorizations) to address this computational challenge, and show how they are naturally realized by convolutional arithmetic circuits. 

\subsection{Shallow Network as a CP Decomposition of $\A^y$} \label{sec:cp_model}

\begin{figure*}
\vspace{-5mm}
\begin{center}
\includegraphics[width=0.7\textwidth]{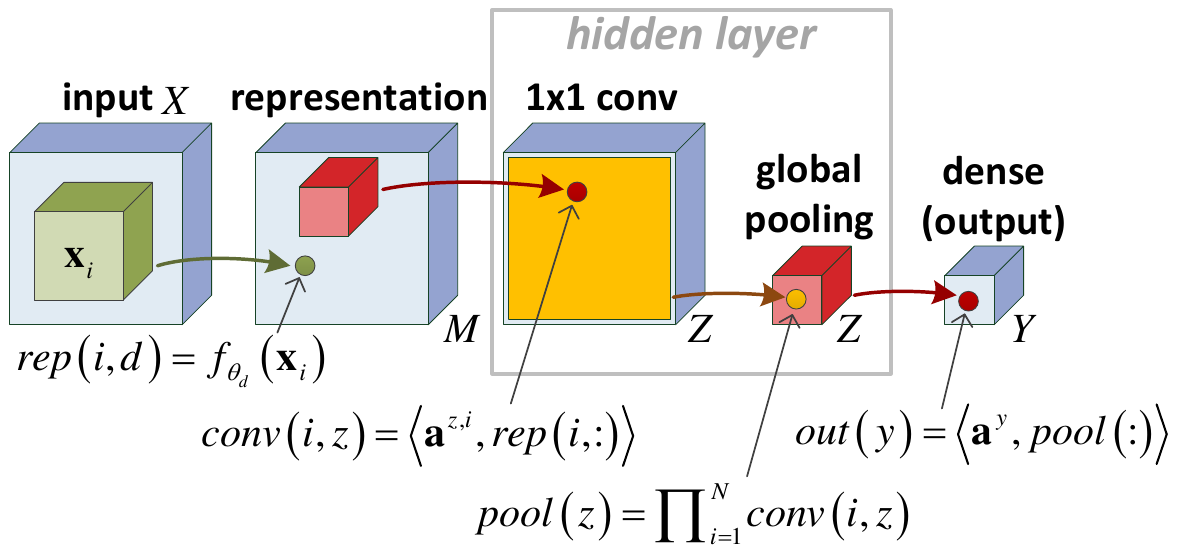}
\end{center}
\vspace{-5mm}
\caption{CP model~--~convolutional arithmetic circuit implementing CP (rank-1) decomposition.}
\label{fig:cp_model}
\end{figure*}

The most straightforward way to factorize a tensor is through a CP (rank-1) decomposition (see sec.~\ref{sec:preliminaries}).
Consider a joint CP decomposition for the coefficient tensors $\{\A^y\}_{y\in\Y}$:
\be
\A^y = \sum_{z=1}^Z a_z^y \cdot \aaa^{z,1} \otimes \cdots \otimes \aaa^{z,N}
\label{eq:cp_decomp}
\ee
where $\aaa^y\in\R^Z$ for $y\in\Y$ ($a_z^y$ stands for entry $z$ of $\aaa^y$), and $\aaa^{z,i}\in\R^M$ for $i\in[N],z\in[Z]$.  
The decomposition is joint in the sense that the same vectors $\aaa^{z,i}$ are shared across all classes $y$.  
Clearly, if we set $Z=M^N$ this model is universal, i.e. any tensors $\A^1{\ldots}\A^Y$ may be represented.  

Substituting our CP decomposition (eq.~\ref{eq:cp_decomp}) into the expression for the score functions in eq.~\ref{eq:score}, we obtain:
$$h_{y}(X)=\sum_{z=1}^{Z}a_z^y\prod_{i=1}^N \left(\sum_{d=1}^M a_d^{z,i} f_{\theta_d}(\x_i)\right)$$
From this we conclude that the network illustrated in fig.~\ref{fig:cp_model} implements a classifier (score functions) under the CP decomposition in eq.~\ref{eq:cp_decomp}.  
We refer to this network as \emph{CP model}.  
The network consists of a representation layer followed by a single hidden layer, which in turn is followed by the output.
The hidden layer begins with a \emph{$1\times1$ conv} operator, which is simply a 3D convolution with $Z$ channels and receptive field $1\times1$.
The convolution may operate without coefficient sharing, i.e.~the filters that generate feature maps by sliding across the previous layer may have different coefficients at different spatial locations.
This is often referred to in the deep learning community as a locally-connected operator (see~\cite{Taigman:2014vs}).
To obtain a standard convolutional operator, simply enforce coefficient sharing by constraining the vectors $\aaa^{z,i}$ in the CP decomposition (eq.~\ref{eq:cp_decomp}) to be equal to each other for different values of~$i$ (this setting is discussed in sec.~\ref{sec:shared}).
Following conv operator, the hidden layer includes global product pooling.
Feature maps generated by conv are reduced to singletons through multiplication of their entries, creating a vector of dimension $Z$.
This vector is then mapped into the $Y$ network outputs through a final dense linear layer.

To recap, CP model (fig.~\ref{fig:cp_model}) is a shallow (single hidden layer) convolutional arithmetic circuit that realizes the CP decomposition (eq.~\ref{eq:cp_decomp}).
It is universal, i.e.~it can realize any coefficient tensors with large enough size ($Z$).
Unfortunately, since the CP-rank of a generic tensor is exponential in its order (see~\cite{Hackbusch-book}), the size required for CP model to be universal is exponential ($Z$ exponential in $N$).

\subsection{Deep Network as a Hierarchical Decomposition of $\A^y$}  \label{sec:ht_model}

\begin{figure*}
\vspace{-5mm}
\begin{center}
\includegraphics[width=\textwidth]{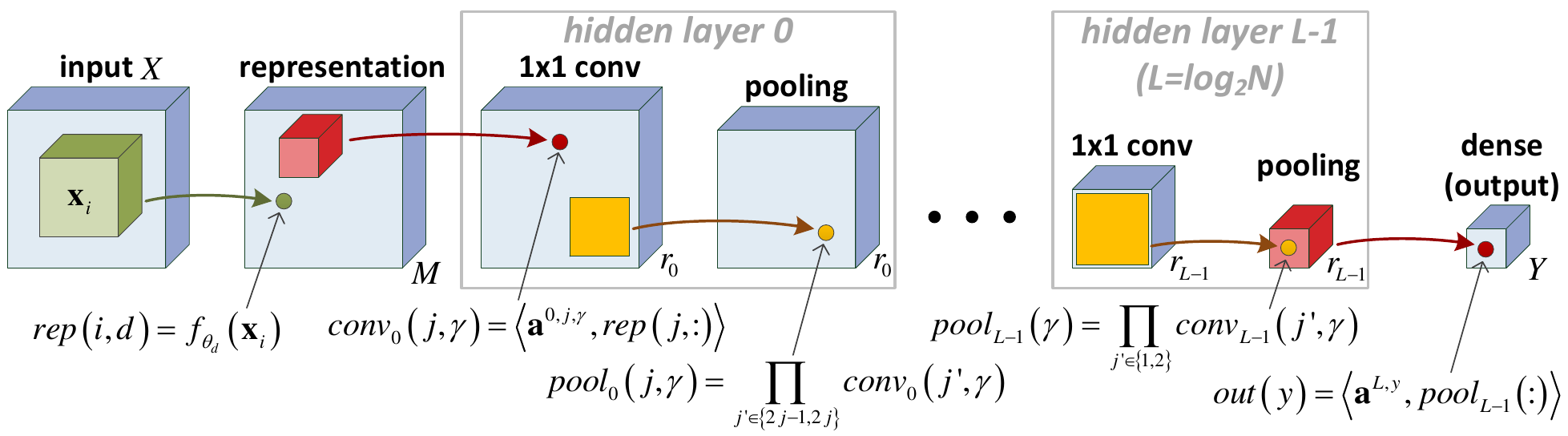}
\end{center}
\vspace{-5mm}
\caption{HT model~--~convolutional arithmetic circuit implementing hierarchical decomposition.}
\label{fig:ht_model}
\end{figure*}

In this subsection we present a deep network that corresponds to the recently introduced Hierarchical Tucker tensor decomposition (\cite{Hackbusch:2009jj}), which we refer to in short as \emph{HT decomposition}.  
The network, dubbed \emph{HT model}, is universal.  
Specifically, any set of tensors $\A^y$ represented by CP model can be represented by HT model with only a polynomial penalty in terms of resources.  
The advantage of HT model, as we show in sec.~\ref{sec:theorems}, is that in almost all cases it generates tensors that require an exponential size in order to be realized, or even approximated, by CP model.  
Put differently, if one draws the weights of HT model by some continuous distribution, with probability one, the resulting tensors cannot be approximated by a polynomial CP model.  
Informally, this implies that HT model is exponentially more expressive than CP model.

\medskip

HT model is based on the hierarchical tensor decomposition in eq.~\ref{eq:ht_decomp}, which is a special case of the HT decomposition as presented in~\cite{Hackbusch:2009jj} (in the latter's terminology, we restrict the matrices $A^{l,j,\gamma}$ to be diagonal).
Our construction and theoretical results apply to the general HT decomposition as well, with the specialization done merely to bring forth a network that resembles current convolutional networks
\footnote{
If we had not constrained $A^{l,j,\gamma}$ to be diagonal, pooling operations would involve entries from different channels.
}.

\bea
\phi^{1,j,\gamma} &=& \sum_{\alpha=1}^{r_0} a_\alpha^{1,j,\gamma} 
\aaa^{0,2j-1,\alpha} \otimes  \aaa^{0,2j,\alpha} 
\nonumber \\
&\cdots& 
\nonumber\\
\phi^{l,j,\gamma} &=& \sum_{\alpha=1}^{r_{l-1}} a_\alpha^{l,j,\gamma} 
\underbrace{\phi^{l-1,2j-1,\alpha}}_{\text{order $2^{l-1}$}} \otimes  
\underbrace{\phi^{l-1,2j,\alpha}}_{\text{order $2^{l-1}$}} 
\nonumber\\
&\cdots& 
\nonumber\\
\phi^{L-1,j,\gamma} &=& \sum_{\alpha=1}^{r_{L-2}} a_\alpha^{L-1,j,\gamma} 
\underbrace{\phi^{L-2,2j-1,\alpha}}_{\text{order $\frac{N}{4}$}} \otimes  
\underbrace{\phi^{L-2,2j,\alpha}}_{\text{order $\frac{N}{4}$}}  
\nonumber\\ 
\A^y &=& \sum_{\alpha=1}^{r_{L-1}} a_\alpha^{L,y} 
\underbrace{\phi^{L-1,1,\alpha}}_{\text{order $\frac{N}{2}$}} \otimes  
\underbrace{\phi^{L-1,2,\alpha}}_{\text{order $\frac{N}{2}$}}  
\label{eq:ht_decomp} 
\eea

The decomposition in eq.~\ref{eq:ht_decomp} recursively constructs the coefficient tensors $\{\A^y\}_{y\in[Y]}$ by assembling vectors $\{\aaa^{0,j,\gamma}\}_{j\in[N],\gamma\in[r_0]}$ into tensors $\{\phi^{l,j,\gamma}\}_{l\in[L-1],j\in[N/2^l],\gamma\in[r_l]}$ in an incremental fashion.  
The index $l$ stands for the level in the decomposition, $j$ represents the ``location'' within level $l$, and $\gamma$ corresponds to the
individual tensor in level $l$ and location $j$.  
$r_l$ is referred to as \emph{level-$l$ rank}, and is defined to be the number of tensors in each location of level $l$ (we denote for completeness $r_L:=Y$).  
The tensor $\phi^{l,j,\gamma}$ has order $2^l$, and we assume for simplicity that $N$~--~the order of $\A^y$, is a power of $2$ (this is merely a technical assumption also made in~\cite{Hackbusch:2009jj}, it does not limit the generality of our analysis).  

The parameters of the decomposition are the final level weights $\{\aaa^{L,y}\in\R^{r_{L-1}}\}_{y\in[Y]}$, the intermediate levels' weights $\{\aaa^{l,j,\gamma} \in \R^{r_{l-1}}\}_{l\in[L-1],j\in[N/2^l],\gamma\in[r_l]}$, and the first level vectors $\{\aaa^{0,j,\gamma}\in\R^M\}_{j\in[N],\gamma\in[r_0]}$.  
This totals at $N \cdot M \cdot r_0 + \sum_{l=1}^{L-1} \frac{N}{2^l} \cdot r_{l-1} \cdot r_l + Y \cdot r_{l-1}$ individual 
parameters, and if we assume equal ranks $r:=r_0 = \cdots = r_{L-1}$, the number of parameters becomes $N \cdot M \cdot r + N \cdot r^2 + Y \cdot r$.

The hierarchical decomposition (eq.~\ref{eq:ht_decomp}) is universal, i.e.~with large enough ranks $r_l$ it can represent any tensors.  
Moreover, it is a super-set of the CP decomposition (eq.~\ref{eq:cp_decomp}).
That is to say, all tensors representable by a CP decomposition having $Z$ components are also representable by a hierarchical decomposition with ranks $r_0=r_1=\cdots=r_{L-1}=Z$ 
\footnote{
To see this, simply assign the first level vectors $\aaa^{0,j,\gamma}$ with CP's basis vectors, the last level weights with CP's per-class weights, and the intermediate levels' weights with indicator vectors.
}.  
Note that this comes with a polynomial penalty~--~the number of parameters increases from $N \cdot M \cdot Z + Z\cdot Y$ in the CP decomposition, to $N \cdot M \cdot Z + Z\cdot Y + N\cdot Z^2$ in the hierarchical decomposition.  
However, as we show in sec.~\ref{sec:theorems}, the gain in expressive power is exponential.

Plugging the expression for $\A^y$ in our hierarchical decomposition (eq.~\ref{eq:ht_decomp}) into the score function $h_y$ given in eq.~\ref{eq:score}, we obtain the network displayed in fig.~\ref{fig:ht_model}~--~HT model. 
This network includes a representation layer followed by $L=\log_{2}N$ hidden layers which in turn are followed by the output.
As in the shallow CP model (fig.~\ref{fig:cp_model}), the hidden layers consist of $1\times1$ conv operators followed by product pooling.
The difference is that instead of a single hidden layer collapsing the entire spatial structure through global pooling, hidden layers now pool over size-$2$ windows, decimating feature maps by a factor of two (no overlaps).
After $L=\log_{2}N$ such layers feature maps are reduced to singletons, and we arrive at a 1D structure with $r_{L-1}$ nodes.
This is then mapped into $Y$ network outputs through a final dense linear layer.
We note that the network's size-$2$ pooling windows (and the resulting number of hidden layers $L=\log_{2}N$) correspond to the fact that our hierarchical decomposition (eq.~\ref{eq:ht_decomp}) is based on a full binary tree over modes, i.e.~it combines (through tensor product) two tensors at a time.
We focus on this setting solely for simplicity of presentation, and since it is the one presented in~\cite{Hackbusch:2009jj}.
Our analysis (sec.~\ref{sec:theorems}) could easily be adapted to hierarchical decompositions based on other trees (taking tensor products between more than two tensors at a time), and that would correspond to networks with different pooling window sizes and resulting depths.

HT model (fig.~\ref{fig:ht_model}) is conceptually divided into two parts.  
The first is the representation layer, transforming input vectors $\x_1{\ldots}\x_N$ into $N{\cdot}M$ real-valued scalars $\{f_{\theta_d}(\x_i)\}_{i\in[N],d\in[M]}$.  
The second and main part of the network, which we view as an ``inference'' engine, is the convolutional arithmetic circuit that takes the $N{\cdot}M$ measurements produced by the representation layer, and accordingly computes $Y$ class scores at the output layer.

To recap, we have now a deep network (fig.~\ref{fig:ht_model}), which we refer to as HT model, that computes the score functions $h_y$ (eq.~\ref{eq:score}) with coefficient tensors $\A^y$ hierarchically decomposed as in eq.~\ref{eq:ht_decomp}.  
The network is universal in the sense that with enough channels $r_l$, any tensors may be represented.  
Moreover, the model is a super-set of the shallow CP model presented in sec.~\ref{sec:cp_model}.  
The question of depth efficiency now naturally arises.
In particular, we would like to know if there are functions that may be represented by a polynomially sized deep HT model, yet require exponential size from the shallow CP model.  
The answer, as described in sec.~\ref{sec:theorems}, is that almost all functions realizable by HT model meet this property.
In other words, the set of functions realizable by a polynomial CP model has \emph{measure zero} in the space of functions realizable by a given polynomial HT model.

\subsection{Shared Coefficients for Convolution} \label{sec:shared}

The $1\times1$ conv operator in our networks (see fig.~\ref{fig:cp_model} and~\ref{fig:ht_model}) implements a local linear transformation with coefficients generally being location-dependent.
In the special case where coefficients do not depend on location, i.e. remain fixed across space, the local linear transformation becomes a standard convolution.
We refer to this setting as coefficient \emph{sharing}.  
Sharing is a widely used structural constraint, one of the pillars behind the successful convolutional network architecture.
In the context of image processing (prominent application of convolutional networks), sharing is motivated by the observation that in natural images, the semantic content of a pattern often does not depend on its location.
In this subsection we explore the effect of sharing on the expressiveness of our networks, or more specifically, on the coefficient tensors $\A^y$ they can represent.

For CP model, coefficient sharing amounts to setting $\aaa^z:=\aaa^{z,1}=\cdots=\aaa^{z,N}$ in the CP decomposition (eq.~\ref{eq:cp_decomp}), transforming the latter to a symmetric CP decomposition:
$$ \A^y = \sum_{z=1}^Z a_z^y \cdot \underbrace{\aaa^z \otimes \cdots \otimes \aaa^z}_{N~\text{times}}
\quad ,\aaa^z \in \R^M,\aaa^y \in \R^Z $$
CP model with sharing is not universal (not all tensors $\A^y$ are representable, no matter how large $Z$ is allowed to be)~--~it can only represent symmetric tensors.  

In the case of HT model, sharing amounts to applying the following constraints on the hierarchical decomposition in eq.~\ref{eq:ht_decomp}: $\aaa^{l,\gamma}:=\aaa^{l,1,\gamma}=\cdots=\aaa^{l,\nicefrac{N}{2^l},\gamma}$ for every $l=0{\ldots}L-1$ and $\gamma=1{\ldots}r_l$.  
Note that in this case universality is lost as well, but nonetheless generated tensors are not limited to be symmetric, already demonstrating an expressive advantage of deep models over shallow ones. 
In sec.~\ref{sec:theorems} we take this further by showing that the shared HT model is exponentially more expressive than CP model, even if the latter is not constrained by sharing.

\section{Theorems of Network Capacity} \label{sec:theorems}

The first contribution of this paper, presented in sec.~\ref{sec:cac}, is the equivalence between deep learning architectures successfully employed in practice, and tensor decompositions.
Namely, we showed that convolutional arithmetic circuits as in fig.~\ref{fig:ht_model}, which are in fact SimNets that have demonstrated promising empirical performance (see app.~\ref{app:simnets}), may be formulated as hierarchical tensor decompositions.
As a second contribution, we make use of the established link between arithmetic circuits and tensor decompositions, combining theoretical tools from these two worlds, to prove results that are of interest to both deep learning and tensor analysis communities.
This is the focus of the current section.

The fundamental theoretical result proven in this paper is the following:
\begin{theorem} \label{thm:fundamental}
Let $\A^y$ be a tensor of order $N$ and dimension $M$ in each mode, generated by the recursive formulas in eq.~\ref{eq:ht_decomp}.  
Define $r:=\min\{r_0,M\}$, and consider the space of all possible configurations for the parameters of the composition~--~$\{\aaa^{l,j,\gamma}\}_{l,j,\gamma}$.  
In this space, the generated tensor $\A^y$ will have CP-rank of at least $r^{\nicefrac{N}{2}}$ almost everywhere (w.r.t. Lebesgue measure).  
Put differently, the configurations for which the CP-rank of $\A^y$ is less than $r^{\nicefrac{N}{2}}$ form a set of measure zero.  
The exact same result holds if we constrain the composition to be ``shared'', i.e. set $\aaa^{l,j,\gamma}\equiv\aaa^{l,\gamma}$ and consider the space of $\{\aaa^{l,\gamma}\}_{l,\gamma}$ configurations.
\end{theorem}

From the perspective of deep learning, thm.~\ref{thm:fundamental} leads to the following corollary:
\begin{corollary} \label{corollary:fundamental}
Given linearly independent representation functions $\{f_{\theta_d}\}_{d\in[M]}$, randomizing the weights of HT model (sec.~\ref{sec:ht_model}) by a continuous distribution induces score functions $h_y$ that with probability one, cannot be approximated arbitrarily well (in $L^2$ sense) by a CP model (sec.~\ref{sec:cp_model}) with less than $\min\{r_0,M\}^{\nicefrac{N}{2}}$ hidden channels.  
This result holds even if we constrain HT model with weight sharing (sec.~\ref{sec:shared}) while leaving CP model in its general form.
\end{corollary}
That is to say, \emph{besides a negligible set, all functions that can be realized by a polynomially sized HT model (with or without weight sharing), require exponential size in order to be realized, or even approximated, by CP model}.
Most of the previous works relating to depth efficiency (see app.~\ref{app:related_work}) merely show \emph{existence} of functions that separate depths (i.e.~that are efficiently realizable by a deep network yet require super-polynomial size from shallow networks).
Corollary~\ref{corollary:fundamental} on the other hand establishes depth efficiency for \emph{almost all} functions that a deep network can implement.
Equally importantly, it applies to deep learning architectures that are being successfully employed in practice (SimNets~--~see app.~\ref{app:simnets}).

Adopting the viewpoint of tensor analysis, thm.~\ref{thm:fundamental} states that besides a negligible set, all tensors realized by HT (Hierarchical Tucker) decomposition cannot be represented by the classic CP (rank-1) decomposition if the latter has less than an exponential number of terms
\footnote{
As stated in sec.~\ref{sec:ht_model}, the decomposition in eq.~\ref{eq:ht_decomp} to which thm.~\ref{thm:fundamental} applies is actually a special case of HT decomposition as introduced in~\cite{Hackbusch:2009jj}.
However, the theorem and its proof can easily be adapted to account for the general case.
We focus on the special case merely because it corresponds to convolutional arithmetic circuit architectures used in practice.
}.
To the best of our knowledge, this result has never been proved in the tensor analysis community.
In the original paper introducing HT decomposition~(\cite{Hackbusch:2009jj}), as a motivating example, the authors present a specific tensor that is efficiently realizable by HT decomposition while requiring an exponential number of terms from CP decomposition
\footnote{
The same motivating example is given in a more recent textbook introducing tensor analysis~(\cite{Hackbusch-book}).
}.
Our result strengthens this motivation considerably, showing that it is not just one specific tensor that favors HT over CP, but rather, almost all tensors realizable by HT exhibit this preference.
Taking into account that any tensor realized by CP can also be realized by HT with only a polynomial penalty in the number of parameters (see sec.~\ref{sec:ht_model}), this implies that in an asymptotic sense, HT decomposition is exponentially more efficient than CP decomposition.

\subsection{Proof Sketches} \label{sec:theorems:proof_sketch}

The complete proofs of thm.~\ref{thm:fundamental} and corollary~\ref{corollary:fundamental} are given in app.~\ref{app:proofs}.
We provide here an outline of the main tools employed and arguments made along these proofs.

\medskip

To prove thm.~\ref{thm:fundamental} we combine approaches from the worlds of circuit complexity and tensor decompositions.
The first class of machinery we employ is \emph{matrix algebra}, which has proven to be a powerful source of tools for analyzing the complexity of circuits.
For example, arithmetic circuits have been analyzed through what is called the partial derivative matrix (see~\cite{raz2009lower}), and for boolean circuits a widely used tool is the communication matrix (see~\cite{karchmer1989communication}).
We gain access to matrix algebra by arranging tensors that take part in the CP and HT decompositions as matrices, a process often referred to as \emph{matricization}.
With matricization, the tensor product translates to the Kronecker product, and the properties of the latter become readily available.
The second tool-set we make use of is \emph{measure theory}, which prevails in the study of tensor decompositions, but is much less frequent in analyses of circuit complexity. 
In order to frame a problem in measure theoretical terms, one obviously needs to define a measure space of interest.
For tensor decompositions, the straightforward space to focus on is that of the decomposition variables.
For general circuits on the other hand, it is often unclear if defining a measure space is at all appropriate.
However, when circuits are considered in the context of machine learning they are usually parameterized, and defining a measure space on top of these parameters is an effective approach for studying the prevalence of various properties in hypotheses spaces.

Our proof of thm.~\ref{thm:fundamental} traverses through the following path.
We begin by showing that matricizing a rank-1 tensor produces a rank-1 matrix.
This implies that the matricization of a tensor generated by a CP decomposition with $Z$ terms has rank at most $Z$.
We then turn to show that the matricization of a tensor generated by the HT decomposition in eq.~\ref{eq:ht_decomp} has rank at least $\min\{r_0,M\}^{N/2}$ almost everywhere.
This is done through induction over the levels of the decomposition ($l=1{\ldots}L$).
For the first level ($l=1$), we use a combination of measure theoretical and linear algebraic arguments to show that the generated matrices have maximal rank ($\min\{r_0,M\}$) almost everywhere.
For the induction step, the facts that under matricization tensor product translates into Kronecker product, and that the latter increases ranks multiplicatively
\footnote{
If $\odot$ denotes the Kronecker product, then for any matrices $A$ and $B$: $rank(A{\odot}B)=rank(A){\cdot}rank(B)$.
},
imply that matricization ranks in the current level are generally equal to those in the previous level squared.
Measure theoretical claims are then made to ensure that this indeed takes place almost everywhere.

\medskip

To prove corollary~\ref{corollary:fundamental} based on thm.~\ref{thm:fundamental}, we need to show that the inability of CP model to realize a tensor generated by HT model, implies that the former cannot approximate score functions produced by the latter.
In general, the set of tensors expressible by a CP decomposition is not topologically closed
\footnote{
Hence the definition of \emph{border rank}, see~\cite{Hackbusch-book}.
},
which implies that a-priori, it may be that CP model can approximate tensors generated by HT model even though it cannot realize them.
However, since the proof of thm.~\ref{thm:fundamental} was achieved through separation of matrix rank, distances are indeed positive and CP model cannot approximate HT model's tensors almost always.
To translate from tensors to score functions, we simply note that in a finite-dimensional Hilbert space convergence in norm implies convergence in coefficients under any basis.
Therefore, in the space of score functions (eq.~\ref{eq:score}) convergence in norm implies convergence in coefficients under the basis $\{(\x_1,\ldots,\x_N){\mapsto}\prod_{i=1}^{N}f_{\theta_{d_i}}(\x_i) \}_{d_1{\ldots}d_N\in[M]}$.
That is to say, it implies convergence in coefficient tensors.

\subsection{Generalization} \label{sec:theorems:general}

Thm.~\ref{thm:fundamental} and corollary~\ref{corollary:fundamental} compare the expressive power of the deep HT model (sec.~\ref{sec:ht_model}) to that of the shallow CP model (sec.~\ref{sec:cp_model}).  
One may argue that such an analysis is lacking, as it does not convey information regarding the importance of each individual layer.  
In particular, it does not shed light on the advantage of very deep networks, which at present provide state of the art recognition accuracy, compared to networks of more moderate depth.  
For this purpose we present a generalization, specifying the amount of resources one has to pay in order to maintain representational power while layers are incrementally cut off from a deep network.  
For conciseness we defer this analysis to app.~\ref{app:general_thm}, and merely state here our final conclusions.
We find that the representational penalty is double exponential w.r.t. the number of layers removed.
In addition, there are certain cases where the removal of even a single layer leads to an exponential inflation, falling in line with the suggestion of~\cite{Bengio:2009kb}.

\section{Discussion} \label{sec:discussion}

In this work we address a fundamental issue in deep learning~--~the expressive efficiency of depth.
There have been many attempts to theoretically analyze this question, but from a practical machine learning perspective, existing results are limited.
Most of the results apply to very specific types of networks that do not resemble ones used in practice, and none of the results account for the locality-sharing-pooling paradigm which forms the basis for convolutional networks~--~the most successful deep learning architecture to date.
In addition, current analyses merely show \emph{existence} of depth efficiency, i.e. of functions that are efficiently realizable by deep networks but not by shallow ones.
The practical implications of such findings are arguably slight, as a-priori, it may be that only a small fraction of the functions realizable by deep networks enjoy depth efficiency, and for all the rest shallow networks suffice.

Our aim in this paper was to develop a theory that facilitates an analysis of depth efficiency for networks that incorporate the widely used structural ingredients of locality, sharing and pooling. 
We consider the task of classification into one of a finite set of categories $\Y=\{1{\ldots}Y\}$.
Our instance space is defined to be the Cartesian product of $N$ vector spaces, in compliance with the common practice of representing natural data through ordered local structures (e.g. images through patches).
Each of the $N$ vectors that compose an instance is represented by a descriptor of length $M$, generated by running the vector through $M$ ``representation'' functions.
As customary, classification is achieved through maximization of score functions $h_y$, one for every category $y\in\Y$.
Each score function is a linear combination over the $M^N$ possible products that may be formed by taking one descriptor entry from every input vector.
The coefficients for these linear combinations conveniently reside in tensors $\A^y$ of order $N$ and dimension $M$ along each axis.
We construct networks that compute score functions $h_y$ by decomposing (factorizing) the coefficient tensors $\A^y$.  
The resulting networks are convolutional arithmetic circuits that incorporate locality, sharing and pooling, and operate on the $N{\cdot}M$ descriptor entries generated from the input.

We show that a shallow (single hidden layer) network realizes the classic CP (rank-1) tensor decomposition, whereas a deep network with $\log_{2}N$ hidden layers realizes the recently introduced Hierarchical Tucker (HT) decomposition~(\cite{Hackbusch:2009jj}).  
Our fundamental result, presented in thm.~\ref{thm:fundamental} and corollary~\ref{corollary:fundamental}, states that randomizing the weights of a deep network by some continuous distribution will lead, \emph{with probability one}, to score functions that cannot be approximated by a shallow network if the latter's size is not exponential (in $N$). 
We extend this result (thm.~\ref{thm:generalized} and corollary~\ref{corollary:generalized}) by deriving analogous claims that compare 
two networks of any depths, not just deep vs. shallow.

To further highlight the connection between our networks and ones used in practice, we show (app.~\ref{app:simnets}) that translating convolution and product pooling computations to log-space (for numerical stability) gives rise to SimNets~--~a recently proposed deep learning architecture which has been shown to produce state of the art accuracy in computationally limited settings~(\cite{simnets2}).

\bigskip

Besides the central line of our work discussed above, the construction and theory presented in this paper shed light on various conjectures and practices employed by the deep learning community.  
First, with respect to the \emph{pooling} operation, our analysis points to the possibility that perhaps it has more to do with factorization of computed functions than it does with translation invariance.
This may serve as an explanation for the fact that pooling windows in state of the art convolutional networks are typically very small (see for example~\cite{simonyan2014very}), often much smaller than the radius of translation one would like to be invariant to.
Indeed, in our framework, as we show in app.~\ref{app:general_thm}, pooling over large windows and trimming down a network's depth may bring to an exponential decrease in expressive efficiency. 

The second point our theory sheds light on is \emph{sharing}.
As discussed in sec.~\ref{sec:shared}, introducing weight sharing to a shallow network (CP model) considerably limits its expressive power.
The network can only represent symmetric tensors, which in turn means that it is location invariant w.r.t. input vectors (patches).
In the case of a deep network (HT model) the limitation posed by sharing is not as strict.  
Generated tensors need not be symmetric, implying that the network is capable of modeling location~--~a crucial ability in almost any real-world task.  
The above findings suggest that the sharing constraint is increasingly limiting as a network gets shallower, to the point where it causes complete ignorance to location.  
This could serve as an argument supporting the empirical success of deep convolutional networks~--~they bind together the statistical and computational advantages of sharing with many layers that mitigate its expressive limitations.

Lastly, our construction advocates \emph{locality}, or more specifically, $1\times1$ receptive fields.  
Recent convolutional networks providing state of the art recognition performance (e.g.~\cite{NiN,Szegedy:2014tb}) make extensive use of $1\times1$ linear transformations, proving them to be very successful in practice.  
In view of our model, such $1\times1$ operators factorize tensors while providing universality with a minimal number of parameters.  
It seems reasonable to conjecture that for this task of factorizing coefficient tensors, larger receptive fields are not significantly helpful, as they lead to redundancy which may deteriorate performance in presence of limited training data.
Investigation of this conjecture is left for future work.

\newcommand{\acknowledgments}
{Amnon Shashua would like to thank Tomaso Poggio and Shai S. Shwartz for illuminating discussions during the preparation of this manuscript.  
We would also like to thank Tomer Galanti, Tamir Hazan and Lior Wolf for commenting on draft versions of the paper.  
The work is partly funded by Intel grant ICRI-CI no. 9-2012-6133 and by ISF Center grant 1790/12.  
Nadav Cohen is supported by a Google Fellowship in Machine Learning.}
\ifdefined\COLT
	\acks{\acknowledgments}
\else
	\ifdefined\CAMREADY
		\subsubsection*{Acknowledgments}
		\acknowledgments
	\fi
\fi

\subsection*{References}
\small{
\bibliographystyle{plainnat}
\bibliography{refs.bib}
}

\clearpage
\appendix

\section{Generalized Theorem of Network Capacity} \label{app:general_thm}

In sec.~\ref{sec:theorems} we presented our fundamental theorem of network capacity (thm.~\ref{thm:fundamental} and corollary~\ref{corollary:fundamental}), showing that besides a negligible set, all functions that can be realized by a polynomially sized HT model (with or without weight sharing), require exponential size in order to be realized, or even approximated, by CP model.
In terms of network depth, CP and HT models represent the extremes~--~the former has only a single hidden layer achieved through global pooling, whereas the latter has $L=\log_{2}N$ hidden layers achieved through minimal (size-$2$) pooling windows.
It is of interest to generalize the fundamental result by establishing a comparison between networks of intermediate depths.
This is the focus of the current appendix.

\medskip

We begin by defining a truncated version of the hierarchical tensor decomposition presented in eq.~\ref{eq:ht_decomp}:
\bea
\phi^{1,j,\gamma} &=& \sum_{\alpha=1}^{r_0} a_\alpha^{1,j,\gamma} \aaa^{0,2j-1,\alpha} \otimes  \aaa^{0,2j,\alpha} 
\nonumber \\
&\vdots& 
\nonumber\\
\phi^{l,j,\gamma} &=& \sum_{\alpha=1}^{r_{l-1}} a_\alpha^{l,j,\gamma} 
\underbrace{\phi^{l-1,2j-1,\alpha}}_{\text{order $2^{l-1}$}} \otimes  \underbrace{\phi^{l-1,2j,\alpha}}_{\text{order $2^{l-1}$}} 
\nonumber\\
&\vdots& 
\nonumber\\
\A  &=& \sum_{\alpha=1}^{r_{L_c - 1}} a^{L_c}_\alpha 
\mathop{\otimes}_{j=1}^{2^{L-L_c+1}} \underbrace{\phi^{L_c-1,j,\alpha}}_{\text{order $2^{L_c - 1}$}} 
\label{eq:ht_decomp_trunc}
\eea
The only difference between this decomposition and the original is that instead of completing the full process with $L:=\log_{2}N$ levels, we stop after $L_c{\leq}L$.  
At this point remaining tensors are binded together to form the final order-$N$ tensor.  
The corresponding network will simply include a premature global pooling stage that shrinks feature maps to $1\times1$, and then a final linear layer that performs classification.  
As before, we consider a shared version of the decomposition in which $\aaa^{l,j,\gamma}\equiv\aaa^{l,\gamma}$.
Notice that this construction realizes a continuum between CP and HT models, which correspond to the extreme cases $L_c=1$ and $L_c=L$ respectively.

The following theorem, a generalization of thm.~\ref{thm:fundamental}, compares a truncated decomposition having $L_1$ levels, to one
with $L_2<L_1$ levels that implements the same tensor, quantifying the penalty in terms of parameters:
\begin{theorem} \label{thm:generalized}
Let $\A^{(1)}$ and $\A^{(2)}$ be tensors of order $N$ and dimension $M$ in each mode, generated by the truncated recursive formulas in eq.~\ref{eq:ht_decomp_trunc}, with $L_1$ and $L_2$ levels respectively.
Denote by $\{r^{(1)}_l\}_{l=0}^{L_1-1}$ and $\{r^{(2)}_l\}_{l=0}^{L_2-1}$ the composition ranks of $\A^{(1)}$ and $\A^{(2)}$ respectively.  
Assuming w.l.o.g. that $L_1 > L_2$, we define $r := \min\{r^{(1)}_0,...,r^{(1)}_{L_2-1},M\}$, and consider the space of all possible configurations for the parameters of $\A^{(1)}$'s composition~--~$\{\aaa^{(1),l,j,\gamma}\}_{l,j,\gamma}$.  
In this space, almost everywhere (w.r.t. Lebesgue measure), the generated tensor $\A^{(1)}$ requires that $r^{(2)}_{L_2-1}\geq(r)^{2^{L - L_2}}$ if one wishes that $\A^{(2)}$ be equal to $\A^{(1)}$.  
Put differently, the configurations for which $\A^{(1)}$ can be realized by $\A^{(2)}$ with $r^{(2)}_{L_2 - 1}<(r)^{2^{L - L_2}}$ form a set of measure zero.  
The exact same result holds if we constrain the composition of $\A^{(1)}$ to be ``shared'', i.e. set $\aaa^{(1),l,j,\gamma}\equiv\aaa^{(1),l,\gamma}$ and consider the space of $\{\aaa^{(1),l,\gamma}\}_{l,\gamma}$ configurations.
\end{theorem}

In analogy with corollary~\ref{corollary:fundamental}, we obtain the following generalization:
\begin{corollary} \label{corollary:generalized}
Suppose we are given linearly independent representation functions $f_{\theta_1}{\ldots}f_{\theta_M}$, and consider two networks that correspond to the truncated hierarchical tensor decomposition in eq.~\ref{eq:ht_decomp_trunc}, with $L_1$ and $L_2$ hidden layers respectively.  
Assume w.l.o.g. that $L_1>L_2$, i.e. that network~1 is deeper than network~2, and define $r$ to be the minimal number of channels across the representation layer and the first $L_2$ hidden layers of network~1.  
Then, if we randomize the weights of network~1 by a continuous distribution, we obtain, with probability one, score functions $h_y$ that cannot be approximated arbitrarily well (in $L^2$ sense) by network~2 if the latter has less than $(r)^{2^{L - L_2}}$ channels in its last hidden layer.  
The result holds even if we constrain network~1 with weight sharing while leaving network~2 in its general form.
\end{corollary}

Proofs of thm.~\ref{thm:generalized} and corollary~\ref{corollary:generalized} are given in app.~\ref{app:proofs}.
Hereafter, we briefly discuss some of their implications.
First, notice that we indeed obtain a generalization of the fundamental theorem of network capacity (thm.~\ref{thm:fundamental} and corollary~\ref{corollary:fundamental}), which corresponds to the extreme case $L_1=L$ and $L_2=1$.  
Second, note that for the baseline case of $L_1=L$, i.e. a full-depth network has generated the target score function, approximating this with a truncated network draws a price that grows \emph{double exponentially} w.r.t. the number of missing layers.
Third, and most intriguingly, we see that when $L_1$ is considerably smaller than $L$, i.e. when a significantly truncated network is sufficient to model our problem, cutting off even a single layer leads to an exponential price, and this price is \emph{independent} of $L_1$.  
Such scenarios of exponential penalty for trimming down a single layer were discussed in~\cite{Bengio:2009kb}, but only in the context of specific functions realized by networks that do not resemble ones used in practice (see~\cite{Hastad91} for an example of such result).  
We prove this in a much broader, more practical setting, showing that for convolutional arithmetic circuit (SimNet~--~see app.~\ref{app:simnets}) architectures, almost any function realized by a significantly truncated network will exhibit this behavior.  
The issue relates to empirical practice, supporting the common methodology of designing networks that go as deep as possible.  
Specifically, it encourages extending network depth by pooling over small regions, avoiding significant spatial decimation that brings network termination closer.

\medskip

We conclude this appendix by stressing once more that our construction and theoretical approach are not limited to the models covered by our theorems (CP model, HT model, truncated HT model).  
These are merely exemplars deemed most appropriate for initial analysis.  
The fundamental and generalized theorems of network capacity are similar in spirit, and analogous theorems for networks with different pooling window sizes and depths (corresponding to different tensor decompositions) may easily be derived.

\section{Proofs} \label{app:proofs}

\subsection{Proof of Theorems~\ref{thm:fundamental} and~\ref{thm:generalized}} \label{app:proofs:thm}

Our proof of thm.~\ref{thm:fundamental} and~\ref{thm:generalized} relies on basic knowledge in measure theory, or more specifically, Lebesgue measure spaces.  
We do not provide here a comprehensive background on this field (the interested reader is referred to~\cite{jones2001lebesgue}), but rather supplement the brief discussion given in sec.~\ref{sec:preliminaries}, with a list of facts we will be using which are not necessarily intuitive:
\begin{itemize}
\item A union of countably (or finitely) many sets of zero measure is itself a set of zero measure.
\item If $p$ is a polynomial over $d$ variables that is not identically zero, the set of points in $\R^d$ in which it 
vanishes has zero measure (see~\cite{caron2005zero} for a short proof of this).
\item If $S \subset \R^{d_1}$ has zero measure, then $S \times \R^{d_2} \subset \R^{d_1+d_2}$, and
every set contained within, have zero measure as well.
\end{itemize}
In the above, and in the entirety of this paper, the only measure spaces we consider are Euclidean spaces equipped with Lebesgue measure.  
Thus when we say that a set of $d$-dimensional points has zero measure, we mean that its Lebesgue measure in the $d$-dimensional Euclidean space is zero.

Moving on to some preliminaries from matrix and tensor theory, we denote by $[\A]$ the \emph{matricization} of an order-$N$ tensor $\A$ (for simplicity, $N$ is assumed to be even), where rows correspond to odd modes and columns correspond to even modes.  
Namely, if $\A\in\R^{M_1{\times\cdots\times}M_N}$, the matrix $[\A]$ has $M_1{\cdot}M_3{\cdot\ldots\cdot}M_{N-1}$ rows and $M_2{\cdot}M_4{\cdot\ldots\cdot}M_N$ columns, rearranging the entries of the tensor such that $\A_{d_1{\ldots}d_N}$ is stored in row index $1+\sum_{i=1}^{\nicefrac{N}{2}}(d_{2i-1}-1)\prod_{j=i+1}^{\nicefrac{N}{2}}M_{2j-1}$ and column index $1+\sum_{i=1}^{\nicefrac{N}{2}}(d_{2i}-1)\prod_{j=i+1}^{\nicefrac{N}{2}}M_{2j}$.  
To distinguish from the tensor product operation $\otimes$, we denote the Kronecker product between matrices by $\odot$.
Specifically, for two matrices $A \in \R^{M_1\times M_2}$ and $B \in \R^{N_1\times N_2}$, $A \odot B$ is the matrix in $\R^{M_1 N_1 \times M_2 N_2}$ that holds $A_{ij} B_{kl}$ in row index $(i-1)N_1+k$ and column index $(j-1)N_2+l$.  
The basic relation that binds together tensor product, matricization and Kronecker product is $[\A\otimes\B]=[\A]\odot[\B]$, where $\A$ and $\B$ are tensors of even orders.  
Two additional facts we will make use of are that the matricization is a linear operator (i.e. for scalars $\alpha_1{\ldots}\alpha_r$ and tensors with the same size $\A_1{\ldots}\A_r$: $[\sum_{i=1}^r\alpha_i\A_i]=\sum_{i=1}^r\alpha_i[\A_i]$), and less trivially, that for any matrices $A$ and $B$, the rank of $A \odot B$ is equal to $rank(A) \cdot rank(B)$ (see~\cite{bellman1970introduction} for a proof).
These two facts, along with the basic relation laid out above, lead to the conclusion that: 
$$ rank\left[\vv^{(z)}_1\otimes\cdots\otimes\vv^{(z)}_{2^L}\right]=\prod_{i=1}^{\nicefrac{2^L}{2}}
rank\overbrace{\left[\vv^{(z)}_{2i-1}\otimes\vv^{(z)}_{2i}\right]}^{\vv^{(z)}_{2i-1}\vv^{(z)\top}_{2i}}=1 $$
and thus:
$$rank\left[\sum_{z=1}^Z\lambda_z\vv^{(z)}_1\otimes\cdots\otimes\vv^{(z)}_{2^L}\right] 
=rank\sum_{z=1}^Z\lambda_z\left[\vv^{(z)}_1\otimes\cdots\otimes\vv^{(z)}_{2^L}\right] 
\leq \sum_{z=1}^Z rank\left[\vv^{(z)}_1\otimes\cdots\otimes\vv^{(z)}_{2^L}\right]=Z$$
In words, an order-$2^L$ tensor given by a CP-decomposition (see sec.~\ref{sec:preliminaries}) with $Z$ terms, has matricization with rank at most $Z$.  
Thus, \emph{to prove that a certain order-$2^L$ tensor has CP-rank of at least $R$, it suffices to show that its matricization has rank of at least $R$}.

\medskip

We now state and prove two lemmas that will be needed for our proofs of thm.~\ref{thm:fundamental} and~\ref{thm:generalized}.

\begin{lemma} \label{lemma:base}
Let $M,N\in\N$, and define the following mapping taking $\x\in\R^{2MN+N}$ to three matrices: $A(\x)\in\R^{M \times N}$, $B(\x)\in\R^{M \times N}$ and $D(\x)\in\R^{N \times N}$.  
$A(\x)$ simply holds the first $MN$ elements of $\x$, $B(\x)$ holds the following $MN$ elements of $\x$, and $D(\x)$ is a 
diagonal matrix that holds the last $N$ elements of $\x$ on its diagonal.  
Define the product matrix $U(\x):=A(\x)D(\x)B(\x)^\top\in\R^{M \times M}$, and consider the set of points $\x$ for which the rank of 
$U(\x)$ is different from $r:=\min\{M,N\}$.  
This set of points has zero measure.  
The result will also hold if the points $\x$ reside in $\R^{MN+N}$, and the same elements are used to assign $A(\x)$ and $B(\x)$ ($A(\x) \equiv B(\x)$).
\end{lemma}
\begin{proof}
Obviously $rank(U(\x)) \leq r$ for all $\x$, so it remains to show that $rank(U(\x)) \geq r$ for all $\x$ but a set of zero measure.  
Let $U_r(\x)$ be the top-left $r \times r$ sub-matrix of $U(\x)$.  
If $U_r(\x)$ is non-singular then of course $rank(U(\x)) \geq r$ as required.  
It thus suffices to show that the set of points $\x$ for which $\det U_r(\x)=0$ has zero measure.  
Now, $\det U_r(\x)$ is a polynomial in the entries of $\x$, and so it either vanishes on a set of zero measure, or it is the zero polynomial (see~\cite{caron2005zero}).  
All that is left is to disqualify the latter option, and that can be done by finding a specific point $\x_0$ for which $\det U_r(\x_0)\neq0$.  
Indeed, we may choose $\x_0$ such that $D(\x_0)$ is the identity matrix and $A(\x_0),B(\x_0)$ hold $1$ on their main diagonal and $0$ otherwise.  
This selection implies that $U_r(\x_0)$ is the identity matrix, and in particular $\det U_r(\x_0)\neq0$.
\end{proof}

\begin{lemma} \label{lemma:step}
Assume we have $p$ continuous mappings from $\R^d$ to $\R^{M \times N}$ taking the point $\y$ to the matrices $A_1(\y){\ldots}A_p(\y)$.  
Assume that under these mappings, the points $\y$ for which every $i\in[p]$ satisfies $rank(A_i(\y))<r$ form a set of zero measure.  
Define a mapping from $\R^p \times \R^d$ to $R^{M \times N}$ given by $(\x,\y) \mapsto A(\x,\y):=\sum_{i=1}^p x_i \cdot A_i(\y)$.  
Then, the points $(\x,\y)$ for which $rank(A(\x,\y))<r$ form a set of zero measure.
\end{lemma}
\begin{proof}
Denote $S:=\{(\x,\y):rank(A(\x,\y))<r\}\subset\R^p \times \R^d$.  
We would like to show that this set has zero measure.  
We first note that since $A(\x,\y)$ is a continuous mapping, and the set of matrices $A \in \R^{M \times N}$ which have rank less than $r$ is closed, $S$ is a closed set and in particular measurable.  
Our strategy for computing its measure will be as follows.  
For every $\y\in\R^d$ we define the marginal set $S^{\y}:=\{\x:rank(A(\x,\y))<r\}\subset\R^p$.  
We will show that for every $\y$ but a set of zero measure, the measure of $S^{\y}$ is zero.  
An application of Fubini's theorem will then prove the desired result.

Let $C$ be the set of points $\y\in\R^d$ for which $\forall i\in[p]:rank(A_i(\y)) < r$.  
By assumption, $C$ has zero measure. 
We now show that for $\y_0 \in \R^d \setminus C$, the measure of $S^{\y_0}$ is zero.
By the definition of $C$ there exists an $i\in[p]$ such that $rank(A_i(\y_0)) \geq r$.  
W.l.o.g., we assume that $i=1$, and that the top-left $r \times r$ sub-matrix of $A_1(\y_0)$ is non-singular.  
Regarding $\y_0$ as fixed, the determinant of the top-left $r \times r$ sub-matrix of $A(\x,\y_0)$ is a polynomial in the elements 
of $\x$.  
It is not the zero polynomial, as setting $x_1=1,x_2=\cdots=x_p=0$ yields $A(\x,\y_0)=A_1(\y_0)$, and the determinant of the latter's top-left $r \times r$ sub-matrix is non-zero.  
As a non-zero polynomial, the determinant of the top-left $r \times r$ sub-matrix of $A(\x,\y_0)$ vanishes only on a set of zero measure~(\cite{caron2005zero}).  
This implies that indeed the measure of $S^{\y_0}$ is zero.

We introduce a few notations towards our application of Fubini's theorem.  
First, the symbol $\1$ will be used to represent indicator functions, e.g. $\1_S$ is the function from $\R^p \times \R^d$ to $\R$ that receives $1$ on $S$ and $0$ elsewhere.  
Second, we use a subscript of $n\in\N$ to indicate that the corresponding set is intersected with the hyper-rectangle of radius $n$.  
For example, $S_n$ stands for the intersection between $S$ and $[-n,n]^{p+d}$, and $\R^d_n$ stands for the intersection between $\R^d$ and $[-n,n]^d$ (which is equal to the latter).  
All the sets we consider are measurable, and those with subscript $n$ have finite measure.  
We may thus apply Fubini's theorem to get:
$$\int_{(\x,\y)} \1_{S_n} = \int_{(\x,\y)\in\R^{p+d}_n} \1_S
= \int_{\y\in\R^d_n} \int_{\x\in\R^p_n} \1_{S^y}
= \int_{\y\in\R^d_n \cap C} \int_{\x\in\R^p_n} \1_{S^y}
+\int_{\y\in\R^d_n \setminus C} \int_{\x\in\R^p_n} \1_{S^y}$$
Recall that the set $C\in\R^d$ has zero measure, and for every $\y \notin C$ the measure of $S^y\in\R^p$ is zero.  
This implies that both integrals in the last expression vanish, and thus $\int \1_{S_n} = 0$.  
Finally, we use the monotone convergence theorem to compute $\int \1_S$:
$$ \int \1_S = \int \lim_{n\to\infty} \1_{S_n} = \lim_{n\to\infty}\int \1_{S_n} = \lim_{n\to\infty} 0 = 0$$
This shows that indeed our set of interest $S$ has zero measure.
\end{proof}

\medskip

With all preliminaries and lemmas in place, we turn to prove thm.~\ref{thm:fundamental}, establishing an exponential efficiency of HT decomposition (eq.~\ref{eq:ht_decomp}) over CP decomposition (eq.~\ref{eq:cp_decomp}).

\medskip

\begin{proof}[\textbf{of theorem~\ref{thm:fundamental}}]
We begin with the case of an ``unshared'' composition, i.e. the one given in eq.~\ref{eq:ht_decomp} (as opposed to the ``shared'' setting of $\aaa^{l,j,\gamma}\equiv\aaa^{l,\gamma}$).  
Denoting for convenience $\phi^{L,1,1}:=\A^y$ and $r_L=1$, we will show by induction over $l=1,...,L$ that almost everywhere (at all points but a set of zero measure) w.r.t. $\{\aaa^{l,j,\gamma}\}_{l,j,\gamma}$, all CP-ranks of the tensors $\{\phi^{l,j,\gamma}\}_{j\in[\nicefrac{N}{2^l}],\gamma\in[r_l]}$ are at least $r^{\nicefrac{2^l}{2}}$.  
In accordance with our discussion in the beginning of this subsection, it suffices to consider the matricizations $[\phi^{l,j,\gamma}]$, and show that these all have ranks greater or equal to $r^{\nicefrac{2^l}{2}}$ almost everywhere.

For the case $l=1$ we have:
$$\phi^{1,j,\gamma} = \sum_{\alpha=1}^{r_0} a_\alpha^{1,j,\gamma} \aaa^{0,2j-1,\alpha} \otimes  \aaa^{0,2j,\alpha}$$
Denote by $A \in \R^{M \times r_0}$ the matrix with columns $\{\aaa^{0,2j-1,\alpha}\}_{\alpha=1}^{r_0}$,
by $B \in \R^{M \times r_0}$ the matrix with columns $\{\aaa^{0,2j,\alpha}\}_{\alpha=1}^{r_0}$, and by
$D \in \R^{r_0 \times r_0}$ the diagonal matrix with $\aaa^{1,j,\gamma}$ on its diagonal.  
Then, we may write $[\phi^{1,j,\gamma}]=ADB^\top$, and according to lemma~\ref{lemma:base} the rank of $[\phi^{1,j,\gamma}]$ equals $r:=\min\{r_0,M\}$ almost everywhere w.r.t. $\left(\{\aaa^{0,2j-1,\alpha}\}_\alpha,\{\aaa^{0,2j,\alpha}\}_\alpha,\aaa^{1,j,\gamma}\right)$.  
To see that this holds almost everywhere w.r.t. $\{\aaa^{l,j,\gamma}\}_{l,j,\gamma}$, one should merely recall that for any dimensions $d_1,d_2 \in \N$, if the set $S \subset \R^{d_1}$ has zero measure, so does any subset of $S \times \R^{d_2} \subset \R^{d_1+d_2}$.  
A finite union of zero measure sets has zero measure, thus the fact that $rank[\phi^{1,j,\gamma}]=r$ holds almost everywhere individually for any $j\in\left[\nicefrac{N}{2}\right]$ and $\gamma\in[r_1]$, implies that it holds almost everywhere jointly for all $j$ and $\gamma$.  
This proves our inductive hypothesis (unshared case) for $l=1$.

Assume now that almost everywhere $rank[\phi^{l-1,j',\gamma'}] \geq r^{\nicefrac{2^{l-1}}{2}}$ for all $j'\in[\nicefrac{N}{2^{l-1}}]$ and $\gamma'\in[r_{l-1}]$.  
For some specific choice of $j\in[\nicefrac{N}{2^l}]$ and $\gamma\in[r_l]$ we have:
$$\phi^{l,j,\gamma} = \sum_{\alpha=1}^{r_{l-1}} a_\alpha^{l,j,\gamma} \phi^{l-1,2j-1,\alpha} \otimes  \phi^{l-1,2j,\alpha}
\implies [\phi^{l,j,\gamma}] = \sum_{\alpha=1}^{r_{l-1}} a_\alpha^{l,j,\gamma} [\phi^{l-1,2j-1,\alpha}] \odot [\phi^{l-1,2j,\alpha}]$$
Denote $M_\alpha := [\phi^{l-1,2j-1,\alpha}] \odot [\phi^{l-1,2j,\alpha}]$ for $\alpha=1{\ldots}r_{l-1}$.
By our inductive assumption, and by the general property $rank(A \odot B)=rank(A){\cdot}rank(B)$, we have that almost everywhere the ranks of all matrices $M_\alpha$ are at least $r^{\nicefrac{2^{l-1}}{2}}\cdot r^{\nicefrac{2^{l-1}}{2}}=r^{\nicefrac{2^l}{2}}$.  
Writing $[\phi^{l,j,\gamma}] = \sum_{\alpha=1}^{r_{l-1}} a_\alpha^{l,j,\gamma} \cdot M_\alpha$, and noticing that $\{M_\alpha\}$ do not depend on $\aaa^{l,j,\gamma}$, we turn our attention to lemma~\ref{lemma:step}.
The lemma tells us that $rank[\phi^{l,j,\gamma}] \geq r^{\nicefrac{2^l}{2}}$ almost everywhere.  
Since a finite union of zero measure sets has zero measure, we conclude that almost everywhere $rank[\phi^{l,j,\gamma}] \geq r^{\nicefrac{2^l}{2}}$ holds jointly for all $j\in[\nicefrac{N}{2^l}]$ and $\gamma\in[r_l]$.  
This completes the proof of the theorem in the unshared case.

Proving the theorem in the shared case may be done in the exact same way, except that for $l=1$ one needs
the version of lemma~\ref{lemma:base} for which $A(\x)$ and $B(\x)$ are equal.
\end{proof}

\medskip

We now head on to prove thm.~\ref{thm:generalized}, which is a generalization of thm.~\ref{thm:fundamental}.  
The proof will be similar in nature to that of thm.~\ref{thm:fundamental}, yet slightly more technical.  
In short, the idea is to show that in the generic case, expressing $\A^{(1)}$ as a sum of tensor products between tensors of order $2^{L_2-1}$ requires at least $r^{\nicefrac{N}{2^{L_2}}}$ terms.  
Since $\A^{(2)}$ is expressed as a sum of $r_{L_2-1}$ such terms, demanding $\A^{(2)}=\A^{(1)}$ implies $r_{L_2-1} \geq r^{\nicefrac{N}{2^{L_2}}}$.

To gain technical advantage and utilize known results from matrix theory (as we did when proving thm.~\ref{thm:fundamental}), we introduce a new tensor ``squeezing'' operator $\varphi$.  
For $q\in\N$, $\varphi_q$ is an operator that receives a tensor with order divisible by $q$, and returns the tensor obtained by merging 
together the latter's modes in groups of size $q$.  
Specifically, when applied to the tensor $\A\in\R^{M_1{\times\cdots\times}M_{c \cdot q}}$ ($c \in \N$), $\varphi_q$ returns a tensor of order $c$ which holds $\A_{d_1{\ldots}d_{c \cdot q}}$ in the location defined by the following index for every mode $t \in [c]$: $1+\sum_{i=1}^{q}(d_{i+q(t-1)}-1)\prod_{j=i+1}^{q}M_{j+q(t-1)}$.  
Notice that when applied to a tensor of order $q$, $\varphi_q$ returns a vector.  
Also note that if $\A$ and $\B$ are tensors with orders divisible by $q$, and $\lambda$ is a scalar, we have the desirable properties:
\begin{itemize}
\item $\varphi_q(\A\otimes\B)=\varphi_q(\A)\otimes\varphi_q(\B)$
\item $\varphi_q(\lambda\A+\B)=\lambda\varphi_q(\A)+\varphi_q(\B)$
\end{itemize}
For the sake of our proof we are interested in the case $q=2^{L_2-1}$, and denote for brevity $\varphi := \varphi_{2^{L_2-1}}$.

As stated above, we would like to show that in the generic case, expressing $\A^{(1)}$ as $\sum_{z=1}^Z \phi^{(z)}_1 \otimes\cdots\otimes \phi^{(z)}_{\nicefrac{N}{2^{L_2-1}}}$, where $\phi^{(z)}_i$ are tensors of order $2^{L_2-1}$, implies $Z \geq r^{\nicefrac{N}{2^{L_2}}}$.  
Applying $\varphi$ to both sides of such a decomposition gives: $\varphi(\A^{(1)})=\sum_{z=1}^Z \varphi(\phi^{(z)}_1) \otimes\cdots\otimes \varphi(\phi^{(z)}_{\nicefrac{N}{2^{L_2-1}}})$, where $\varphi(\phi^{(z)}_i)$ are now vectors.  
Thus, to prove thm.~\ref{thm:generalized} it suffices to show that in the generic case, the CP-rank of $\varphi(\A^{(1)})$ is at least $r^{\nicefrac{N}{2^{L_2}}}$, or alternatively, that the rank of the matricization $[\varphi(\A^{(1)})]$ is at least $r^{\nicefrac{N}{2^{L_2}}}$.  
This will be our strategy in the following proof:

\medskip

\begin{proof}[\textbf{of theorem~\ref{thm:generalized}}]
In accordance with the above discussion, it suffices to show that in the generic case $rank[\varphi(\A^{(1)})] \geq r^{\nicefrac{N}{2^{L_2}}}$.  
To ease the path for the reader, we reformulate the problem using slightly simpler notations.
We have an order-$N$ tensor $\A$ with dimension $M$ in each mode, generated as follows:
\beas
\phi^{1,j,\gamma} &=& \sum_{\alpha=1}^{r_0} a_\alpha^{1,j,\gamma} \aaa^{0,2j-1,\alpha} \otimes  \aaa^{0,2j,\alpha} \\
&\vdots& \\
\phi^{l,j,\gamma} &=& \sum_{\alpha=1}^{r_{l-1}} a_\alpha^{l,j,\gamma} \underbrace{\phi^{l-1,2j-1,\alpha}}_{\text{order $2^{l-1}$}} \otimes  \underbrace{\phi^{l-1,2j,\alpha}}_{\text{order $2^{l-1}$}} \\
&\vdots& \\
\A  &=& \sum_{\alpha=1}^{r_{L_1 - 1}} a^{L_1,1,1}_\alpha \mathop{\otimes}_{j=1}^{2^{L-L_1+1}} \underbrace{\phi^{L_1-1,j,\alpha}}_{\text{order $2^{L_1 - 1}$}}
\eeas
where:
\begin{itemize}
\item $L_1 \leq L := \log_2 N$
\item $r_0,...,r_{L_1 - 1}\in\N_{>0}$
\item $\aaa^{0,j,\alpha} \in \R^M$ for $j\in[N]$ and $\alpha\in[r_0]$
\item $\aaa^{l,j,\gamma}\in\R^{r_{l-1}}$ for $l\in[L_1-1]$, $j\in[\nicefrac{N}{2^l}]$ and $\gamma\in[r_l]$
\item $\aaa^{L_1,1,1}\in\R^{r_{L_1 - 1}}$
\end{itemize}
Let $L_2$ be a positive integer smaller than $L_1$, and let $\varphi$ be the tensor squeezing operator that merges groups of $2^{L_2-1}$ modes.  
Define $r:=\min\{r_0,...,r_{L_2-1},M\}$.  
With $[\cdot]$ being the matricization operator defined in the beginning of the appendix, our task is to prove that $rank[\varphi(\A)] \geq r^{\nicefrac{N}{2^{L_2}}}$ almost everywhere w.r.t. $\{\aaa^{l,j,\gamma}\}_{l,j,\gamma}$.  
We also consider the case of shared parameters~--~$\aaa^{l,j,\gamma}\equiv\aaa^{l,\gamma}$, where we would like to show that the same condition holds almost everywhere w.r.t. $\{\aaa^{l,\gamma}\}_{l,\gamma}$.

Our strategy for proving the claim is inductive.
We show that for $l=L_2{\ldots}L_1-1$, almost everywhere it holds that for all $j$ and all $\gamma$: $rank[\varphi(\phi^{l,j,\gamma})] \geq r^{2^{l-L_2}}$.  
We then treat the special case of $l=L_1$, showing that indeed $rank[\varphi(\A)] \geq r^{\nicefrac{N}{2^{L_2}}}$.  
We begin with the setting of unshared parameters ($\aaa^{l,j,\gamma}$), and afterwards attend the scenario of shared parameters ($\aaa^{l,\gamma}$) as well.

Our first task is to treat the case $l=L_2$, i.e. show that $rank[\varphi(\phi^{L_2,j,\gamma})] \geq r$ almost everywhere jointly for all $j$ and all $\gamma$ (there is actually no need for the matricization $[\cdot]$ here, as $\varphi(\phi^{L_2,j,\gamma})$ are already matrices).  
Since a union of finitely many zero measure sets has zero measure, it suffices to show that this condition holds almost everywhere when specific $j$ and $\gamma$ are chosen.  
Denote by $\e_i$ a vector holding $1$ in entry $i$ and $0$ elsewhere, by $\0$ a vector of zeros, and by $\1$ a vector of ones.  
Suppose that for every $j$ we assign $\aaa^{0,j,\alpha}$ to be $\e_\alpha$ when $\alpha \leq r$ and $\0$ otherwise.  
Suppose also that for all $1 \leq l \leq L_2-1$ and all $j$ we set $\aaa^{l,j,\gamma}$ to be $\e_\gamma$ when $\gamma \leq r$ and $\0$ otherwise.  
Finally, assume we set $\aaa^{L_2,j,\gamma}=\1$ for all $j$ and all $\gamma$.  
These settings imply that for every $j$, when $\gamma \leq r$ we have $\phi^{L_2-1,j,\gamma}=\mathop{\otimes}_{j=1}^{2^{L_2-2}} 
(\e_\gamma \otimes \e_\gamma)$, i.e. the tensor $\phi^{L_2-1,j,\gamma}$ holds $1$ in location $(\gamma,...,\gamma)$ and $0$ elsewhere.  
If $\gamma > r$ then $\phi^{L_2-1,j,\gamma}$ is the zero tensor.  
We conclude from this that there are indices $1 \leq i_1 < ... < i_r \leq M^{L_2-1}$ such that $\varphi(\phi^{L_2-1,j,\gamma})=\e_{i_\gamma}$ for $\gamma \leq r$, and that for $\gamma > r$ we have $\varphi(\phi^{L_2-1,j,\gamma})=\0$.  
We may thus write:
$$\varphi(\phi^{L_2,j,\gamma}) 
= \varphi\left(\sum_{\alpha=1}^{r_{L_2-1}}\phi^{L_2-1,2j-1,\alpha} \otimes  \phi^{L_2-1,2j,\alpha} \right) 
= \sum_{\alpha=1}^{r_{L_2-1}} \varphi(\phi^{L_2-1,2j-1,\alpha}) \otimes  \varphi(\phi^{L_2-1,2j,\alpha}) 
= \sum_{\alpha=1}^r\e_{i_\alpha}\e_{i_\alpha}^\top$$
Now, since $i_1{\ldots}i_r$ are different from each other, the matrix $\varphi(\phi^{L_2,j,\gamma})$ has rank $r$.  
This however does not prove our inductive hypothesis for $l=L_2$.  
We merely showed a specific parameter assignment for which it holds, and we need to show that it is met almost everywhere.  
To do so, we consider an $r \times r$ sub-matrix of $\varphi(\phi^{L_2,j,\gamma})$ which is non-singular under the specific parameter
assignment we defined.  
The determinant of this sub-matrix is a polynomial in the elements of $\{\aaa^{l,j,\gamma}\}_{l,j,\gamma}$ which we know does not vanish with the specific assignments defined.
Thus, this polynomial vanishes at subset of $\{\aaa^{l,j,\gamma}\}_{l,j,\gamma}$ having zero measure (see~\cite{caron2005zero}).  
That is to say, the sub-matrix of $\varphi(\phi^{L_2,j,\gamma})$ has rank $r$ almost everywhere, and thus $\varphi(\phi^{L_2,j,\gamma})$ has rank at least $r$ almost everywhere.  
This completes our treatment of the case $l=L_2$.

We now turn to prove the propagation of our inductive hypothesis.  
Let $l\in\{L_2+1,...,L_1-1\}$, and assume that our inductive hypothesis holds for $l-1$.  
Specifically, assume that almost everywhere w.r.t. $\{\aaa^{l,j,\gamma}\}_{l,j,\gamma}$, we have that $rank[\varphi(\phi^{l-1,j,\gamma})] \geq r^{2^{l-1-L_2}}$ jointly for all $j\in[\nicefrac{N}{2^{l-1}}]$ and all $\gamma\in[r_{l-1}]$.  
We would like to show that almost everywhere, $rank[\varphi(\phi^{l,j,\gamma})] \geq r^{2^{l-L_2}}$ jointly for all $j\in[\nicefrac{N}{2^l}]$ and all $\gamma\in[r_l]$.  
Again, the fact that a finite union of zero measure sets has zero measure implies that we may prove the condition for specific $j\in[\nicefrac{N}{2^l}]$ and $\gamma\in[r_l]$.  
Applying the squeezing operator $\varphi$ followed by matricization $[\cdot]$ to the recursive expression for $\phi^{l,j,\gamma}$, we get:
\beas
[\varphi(\phi^{l,j,\gamma})] &=& \left[\varphi\left(\sum_{\alpha=1}^{r_{l-1}} a_\alpha^{l,j,\gamma} \phi^{l-1,2j-1,\alpha} \otimes  \phi^{l-1,2j,\alpha}\right)\right] 
   =   \left[\sum_{\alpha=1}^{r_{l-1}} a_\alpha^{l,j,\gamma} \varphi(\phi^{l-1,2j-1,\alpha}) \otimes  \varphi(\phi^{l-1,2j,\alpha})\right] \\
 &=& \sum_{\alpha=1}^{r_{l-1}} a_\alpha^{l,j,\gamma} [\varphi(\phi^{l-1,2j-1,\alpha})] \odot  [\varphi(\phi^{l-1,2j,\alpha})]
\eeas
For $\alpha=1{\ldots}r_{l-1}$, denote the matrix $[\varphi(\phi^{l-1,2j-1,\alpha})] \odot [\varphi(\phi^{l-1,2j,\alpha})]$ by $M_\alpha$.  
The fact that the Kronecker product multiplies ranks, along with our inductive assumption, imply that almost everywhere $rank(M_\alpha) \geq r^{2^{l-1-L_2}} \cdot r^{2^{l-1-L_2}} = r^{2^{l-L_2}}$.
Noting that the matrices $M_\alpha$ do not depend on $\aaa^{l,j,\gamma}$, we apply lemma~\ref{lemma:step} and conclude that almost everywhere $rank[\varphi(\phi^{l,j,\gamma})] \geq r^{2^{l-L_2}}$, which completes the prove of the inductive propagation.

Next, we treat the special case $l=L_1$.  
We assume now that almost everywhere $rank[\varphi(\phi^{L_1-1,j,\gamma})] \geq r^{2^{L_1-1-L_2}}$ jointly for all $j$ and all $\gamma$.
Again, we apply the squeezing operator $\varphi$ followed by matricization $[\cdot]$, this time to both sides of the expression for $\A$:
$$ [\varphi(\A)]  = \sum_{\alpha=1}^{r_{L_1 - 1}} a^{L_1,1,1}_\alpha \mathop{\odot}_{j=1}^{2^{L-L_1+1}} [\varphi(\phi^{L_1-1,j,\alpha})] $$
As before, denote $M_\alpha := \mathop{\odot}_{j=1}^{2^{L-L_1+1}} [\varphi(\phi^{L_1-1,j,\alpha})]$ for $\alpha=1{\ldots}r_{L_1-1}$.  
Using again the multiplicative rank property of the Kronecker product along with our inductive assumption, we get that almost everywhere $rank(M_\alpha) \geq \prod_{j=1}^{2^{L-L_1+1}}r^{2^{L_1-1-L_2}} = r^{L-L_2}$.  
Noticing that $\{M_\alpha\}_{\alpha\in[r_{L_1-1}]}$ do not depend on $\aaa^{L_1,1,1}$, we apply lemma~\ref{lemma:step} for the last time and get that almost everywhere (w.r.t. $\{\aaa^{l,j,\gamma}\}_{l,j,\gamma}$), the rank of $[\varphi(\A)]$ is at least $r^{L-L_2}$.  This completes our proof in the case of unshared parameters.

Proving the theorem in the case of shared parameters ($\aaa^{l,j,\gamma} \equiv \aaa^{l,\gamma}$) can be done in the exact same way as above.  
In fact, all one has to do is omit the references to $j$ and the proof will apply.  
Notice in particular that the specific parameter assignment we defined to handle $l=L_2$ was completely symmetric, i.e. it did not include any dependence on $j$.
\end{proof}

\subsection{Proof of Corollaries~\ref{corollary:fundamental} and~\ref{corollary:generalized}} \label{app:proofs:corollary}

Corollaries~\ref{corollary:fundamental} and~\ref{corollary:generalized} are a direct continuation of thm.~\ref{thm:fundamental} and~\ref{thm:generalized} respectively.
In the theorems, we have shown that almost all coefficient tensors generated by a deep network cannot be realized by a shallow network if the latter does not meet a certain minimal size requirement.
The corollaries take this further, by stating that given linearly independent representation functions $f_{\theta_1}{\ldots}f_{\theta_M}$, not only is efficient realization of coefficient tensors generally impossible, but also efficient approximation of score functions.
To prove this extra step, we recall from the proofs of thm.~\ref{thm:fundamental} and~\ref{thm:generalized} (app.~\ref{app:proofs:thm}) that in order to show separation between the coefficient tensor of a deep network and that of a shallow network, we relied on matricization rank.
Specifically, we derived constants $R^D,R^S\in\N$, $R^D>R^S$, such that the matricization of a deep network's coefficient tensor had rank greater or equal to $R^D$, whereas the matricization of a shallow network's coefficient tensor had rank smaller or equal to $R^S$.
Given this observation, corollaries~\ref{corollary:fundamental} and~\ref{corollary:generalized} readily follow from lemma~\ref{lemma:approx} below (the lemma relies on basic concepts and results from the topic of $L^2$ Hilbert spaces~--~see app.~\ref{app:hypo_space:preliminaries} for a brief discussion on the matter).

\begin{lemma} \label{lemma:approx}
Let $f_{\theta_1}{\ldots}f_{\theta_M}{\in}L^2(\R^s)$ be a set of linearly independent functions, and denote by $\T$ the (Euclidean) space of tensors with order $N$ and dimension $M$ in each mode.  
For a given tensor $\A\in\T$, denote by $h(\A)$ the function in $L^2\left((\R^s)^N\right)$ defined by:
$$(\x_1,\ldots,\x_N) \overset{h(\A)}{\mapsto}\sum_{d_1,\ldots,d_N=1}^M\A_{d_1{\ldots}d_N}\prod_{i=1}^{N}f_{\theta_{d_i}}(\x_i)$$
Let $\{\A^\lambda\}_{\lambda\in\Lambda}\subset\T$ be a family of tensors, and $\A^*$ be a certain target tensor that lies outside the family.
Assume that for all $\lambda\in\Lambda$ we have $rank([\A^\lambda])<rank([\A^*])$, where $[\cdot]$ is the matricization operator defined in app.~\ref{app:proofs:thm}.
Then, the distance in $L^2\left((\R^s)^N\right)$ between $h(\A^*)$ and $\{h(\A^\lambda)\}_{\lambda\in\Lambda}$ is strictly 
positive, i.e. there exists an $\epsilon>0$ such that:
$$\forall{\lambda\in\Lambda}: \int \abs{h(\A^\lambda)-h(\A^*)}^2 > \epsilon $$
\end{lemma}
\begin{proof}
The fact that $\{f_{\theta_d}(\x)\}_{d\in[M]}$ are linearly independent in $L^2(\R^s)$ implies that the product functions $\{\prod_{i=1}^N f_{\theta_{d_i}}(\x_i) \}_{d_1{\ldots}d_N\in[M]}$ are linearly independent in $L^2\left((\R^s)^N\right)$ (see app.~\ref{app:hypo_space:preliminaries}).  
Let $(h^{(t)})_{t=1}^\infty$ be a sequence of functions that lie in the span of $\{\prod_{i=1}^{N}f_{\theta_{d_i}}(\x_i)\}_{d_1{\ldots}d_N\in[M]}$, and for every $t\in\N$ denote by $\A^{(t)}$ the coefficient tensor of $h^{(t)}$ under this basis, i.e. $\A^{(t)}\in\T$ is defined by:
$$h^{(t)}\left(\x_1,\ldots,\x_N\right)=\sum_{d_1,\ldots,d_N=1}^M \A_{d_1,\ldots,d_N}^{(t)}\prod_{i=1}^{N} f_{\theta_{d_i}}(\x_i)$$
Assume that $(h^{(t)})_{t=1}^\infty$ converges to $h(\A^*)$ in $L^2\left((\R^s)^N\right)$:
$$\lim_{t\to\infty}\int\abs{h^{(t)}-h(\A^*)}^2=0$$
In a finite-dimensional Hilbert space, convergence in norm implies convergence in representation coefficients under any preselected basis.  
We thus have:
$$\forall{d_1{\ldots}d_N\in[M]}:\A_{d_1,\ldots,d_N}^{(t)} \xrightarrow{t\to\infty} \A_{d_1,\ldots,d_N}^*$$
This means in particular that in the tensor space $\T$, $\A^*$ lies in the closure of $\{\A^{(t)}\}_{t=1}^\infty$.
Accordingly, in order to show that the distance in $L^2\left((\R^s)^N\right)$ between $h(\A^*)$ and $\{h(\A^\lambda)\}_{\lambda\in\Lambda}$ is strictly positive, it suffices to show that the distance in $\T$ between $\A^*$ and
$\{\A^\lambda\}_{\lambda\in\Lambda}$ is strictly positive, or equivalently, that the distance between the matrix $[\A^*]$ and the family of matrices $\{[\A^\lambda]\}_{\lambda\in\Lambda}$ is strictly positive.
This however is a direct implication of the assumption $\forall{\lambda\in\Lambda}:rank([\A^\lambda])<rank([\A^*])$.
\end{proof}

\section{Derivation of Hypotheses Space} \label{app:hypo_space}

In order to keep the body of the paper at a reasonable length, the presentation of our hypotheses space (eq.~\ref{eq:score}) in sec.~\ref{sec:cac} did not provide the grounds for its definition.
In this appendix we derive the hypotheses space step by step.
After establishing basic preliminaries on the topic of $L^2$ spaces, we utilize the notion of tensor products between such spaces to reach a universal representation as in eq.~\ref{eq:score} but with $M\to\infty$.
We then make use of empirical studies characterizing the statistics of natural images, to argue that in practice a moderate value of~$M$ ($M\in\Omega(100)$) suffices.

\subsection{Preliminaries on $\mathbf{L^2}$ Spaces} \label{app:hypo_space:preliminaries}

When dealing with functions over scalars, vectors or collections of vectors, we consider $L^2$ spaces, or more formally, the Hilbert spaces
of Lebesgue measurable square-integrable real functions equipped with standard (point-wise) addition and scalar multiplication, as well as the inner-product defined by integral over point-wise multiplication.
The topic of $L^2$ function spaces lies at the heart of functional analysis, and requires basic knowledge in measure theory.  
We present here the bare necessities required to follow this appendix, referring the interested reader to~\cite{rudin1991functional} for a more comprehensive introduction. 

For our purposes, it suffices to view an $L^2$ space as a vector space of all functions $f$ satisfying $\int f^2<\infty$.  
This vector space is infinite dimensional, and a set of functions $\F \subset L^2$ is referred to as \emph{total} if the closure of its span covers the entire space, i.e. if for any function $g \in L^2$ and $\epsilon>0$, there exist functions $f_1{\ldots}f_K\in\F$ and coefficients $c_1{\ldots}c_K\in\R$ such that $\int |\sum_{i=1}^K c_i \cdot f_i - g|^2<\epsilon$.  
$\F$ is regarded as \emph{linearly independent} if all of its finite subsets are linearly independent, i.e. for any $f_1{\ldots}f_K\in\F$, $f_i\neq f_j$, and $c_1{\ldots}c_K\in\R$, if $\sum_{i=1}^K c_i \cdot f_i = 0$ then $c_1 = \cdots = c_K = 0$.  
A non-trivial result states that $L^2$ spaces in general must contain total and linearly independent sets, and moreover, for any $s \in \N$, $L^2(\R^s)$ contains a \emph{countable} set of this type.  
It seems reasonable to draw an analogy between total and linearly independent sets in $L^2$ space, and bases in a finite dimensional vector space.  
While this analogy is indeed appropriate from our perspective, total and linearly independent sets are not to be confused with \emph{bases} for $L^2$ spaces, which are typically defined to be orthonormal.

It can be shown (see for example~\cite{Hackbusch-book}) that for any natural numbers $s$ and $N$, if $\{f_d(\x)\}_{d\in\N}$ is a total or a linearly independent set in $L^2(\R^s)$, then $\{(\x_1,\ldots,\x_N)\mapsto\prod_{i=1}^N f_{d_i}(\x_i) \}_{d_1{\ldots}d_N\in\N}$, the induced point-wise product functions on $(\R^s)^N$, form a set which is total or linearly independent, respectively, in $L^2\left((\R^s)^N\right)$.  
As we now briefly outline, this result actually emerges from a deep relation between tensor products and Hilbert spaces.  
The definitions given in sec.~\ref{sec:preliminaries} for a tensor, tensor space, and tensor product, are actually concrete special cases of much deeper, abstract algebraic concepts.  
A more formal line of presentation considers multiple vector spaces $V_1{\ldots}V_N$, and defines their tensor product space $V_1{\otimes\cdots\otimes}V_N$ to be a specific quotient space of the space freely generated by their Cartesian product set.  
For every combination of vectors $\vv^{(i)} \in V_i$, $i\in[N]$, there exists a corresponding element $\vv^{(1)}{\otimes\cdots\otimes}\vv^{(N)}$ in the tensor product space, and moreover, elements of this form span the entire space.  
If $V_1{\ldots}V_N$ are Hilbert spaces, it is possible to equip $V_1{\otimes\cdots\otimes}V_N$ with a natural inner-product operation, thereby turning it too into a Hilbert space.  
It may then be shown that if the sets $\{\vv^{(i)}_\alpha\}_\alpha \subset V_i$, $i\in[N]$, are total or linearly independent, elements of the form $\vv^{(1)}_{\alpha_1}\otimes\cdots\otimes\vv^{(N)}_{\alpha_N}$ are total or linearly independent, respectively, in $V_1{\otimes\cdots\otimes}V_N$.  
Finally, when the underlying Hilbert spaces are $L^2(\R^s)$, the point-wise product mapping $f_1(\x){\otimes\cdots\otimes}f_N(\x)\mapsto\prod_{i=1}^N f_i(\x_i)$ from the tensor product space $\left(L^2(\R^s)\right)^{\otimes N}:=L^2(\R^s){\otimes\cdots\otimes}L^2(\R^s)$ to $L^2\left((\R^s)^N\right)$, induces an isomorphism of Hilbert spaces.

\subsection{Construction} \label{app:hypo_space:construction}

Recall from sec.~\ref{sec:cac} that our instance space is defined as $\X:=(\R^s)^N$, in accordance with the common practice of representing natural data through ordered local structures (for example images are often represented through small patches around their pixels).
We classify instances into categories $\Y:=\{1{\ldots}Y\}$ via maximization of per-label score functions $\{h_y:(\R^s)^N\to\R\}_{y\in\Y}$.
Our hypotheses space $\HH$ is defined to be the subset of $L^2\left((\R^s)^N\right)$ from which score functions may be taken.

In app.~\ref{app:hypo_space:preliminaries} we stated that if $\{f_d(\x)\}_{d\in\N}$ is a total set in $L^2(\R^s)$, i.e. if every function in $L^2(\R^s)$ can be arbitrarily well approximated by a linear combination of a finite subset of $\{f_d(\x)\}_{d\in\N}$, then the point-wise products $\{(\x_1,\ldots,\x_N)\mapsto\prod_{i=1}^N f_{d_i}(\x_j) \}_{d_1,\ldots,d_N\in\N}$ form a total set in $L^2\left((\R^s)^N\right)$.  
Accordingly, in a universal hypotheses space $\HH=L^2\left((\R^s)^N\right)$, any score function $h_y$ may be arbitrarily well approximated by finite linear combinations of such point-wise products.  
A possible formulation of this would be as follows.  
Assume we are interested in $\epsilon$-approximation of the score function $h_y$, and consider a formal tensor $\A^y$ having $N$ modes and a countable infinite dimension in each mode $i\in[N]$, indexed by $d_i\in\N$.  
Then, there exists such a tensor, with all but a finite number of entries set to zero, for which:
\be
h_{y}\left(\x_1,\ldots,\x_N\right) \approx \sum_{d_1{\ldots}d_N\in\N} A_{d_1,\ldots,d_N}^{y}\prod_{i=1}^{N} f_{d_i}(\x_i)
\label{eq:score_infinity}
\ee

Given that the set of functions $\{f_d(\x)\}_{d\in\N}{\subset}L^2(\R^s)$ is total, eq.~\ref{eq:score_infinity} defines a universal hypotheses space.
There are many possibilities for choosing a total set of functions.  
Wavelets are perhaps the most obvious choice, and were indeed used in a deep network setting by~\cite{Bruna:2012vu}.  
The special case of Gabor wavelets has been claimed to induce features that resemble representations in the visual cortex~(\cite{Serre:2005bd}).  
Two options we pay special attention to due to their importance in practice are:
\begin{itemize}
\item
\emph{Gaussians} (with diagonal covariance):
\be
f_\theta(\x) = \mathcal{N} \left(\x;\mubf,diag(\sigmabf^2) \right)
\label{eq:rep_gaussians}
\ee
where $\theta=(\mubf\in\R^s,\sigmabf^2\in\R^s_{++})$.
\item 
\emph{Neurons}:
\be
f_\theta(\x) = \sigma\left(\x^\top\w+b\right)
\label{eq:rep_neurons}
\ee
where $\theta=(\w\in\R^s,b\in\R)$ and $\sigma$ is a point-wise non-linear activation such as threshold $\sigma(z)=\indc{z>0}$, rectified linear unit (ReLU) $\sigma(z)=\max\{z,0\}$ or sigmoid $\sigma(z)=1/(1+e^{-z})$.
\end{itemize}
In both cases, there is an underlying parametric family of functions $\F=\left\{ f_{\theta}:\R^{s}\to\R\right\} _{\theta\in\Theta}$ of which a countable total subset may be chosen.  
The fact that Gaussians as above are total in $L^2(\R^s)$ has been proven in~\cite{Girosi:1990gf}, and is a direct corollary of the Stone-Weierstrass theorem.  
To achieve countability, simply consider Gaussians with rational parameters (mean and variances).  
In practice, the choice of Gaussians (with diagonal covariance) give rises to a ``similarity'' operator as described by the SimNet architecture~(\cite{simnets1,simnets2}).
For the case of neurons we must restrict the domain $\R^s$ to some bounded set, otherwise the functions are not integrable.  
This however is not a limitation in practice, and indeed neurons are widely used across many application domains.  
The fact that neurons are total has been proven in~\cite{Cybenko:1989fm} and~\cite{Hornik:1989fr} for threshold and sigmoid activations.  
More generally, it has been proven in~\cite{Stinchcombe:1989kw} for a wide class of activation functions, including linear combinations of ReLU.  
See~\cite{Pinkus:1999gk} for a survey of such results.  
For countability, we may again restrict parameters (weights and bias) to be rational.

In the case of Gaussians and neurons, we argue that a \emph{finite} set of functions suffices, i.e. that it is possible to choose $f_{\theta_1}{\ldots}f_{\theta_M}\in\F$ that will suffice in order to represent score functions required for natural tasks.  
Moreover, we claim that $M$ need not be large (e.g. on the order of~100).  
Our argument relies on statistical properties of natural images, and is fully detailed in app.~\ref{app:hypo_space:finite_bases}.  
It implies that under proper choice of $\{f_{\theta_d}(\x)\}_{d\in[M]}$, the finite set of point-wise product functions 
$\{(\x_1,\ldots,\x_N)\mapsto\prod_{i=1}^N f_{\theta_{d_i}}(\x_i) \}_{d_1,\ldots,d_N\in[M]}$ spans the score functions of interest, and we may define for each label $y$ a tensor $\A^y$ of order $N$ and dimension $M$ in each mode, such that:
\be
h_{y}\left(\x_1,\ldots,\x_N\right)=\sum_{d_1,\ldots,d_N=1}^M\A_{d_1,\ldots,d_N}^{y}\prod_{i=1}^{N} f_{\theta_{d_i}}(\x_i)
\tag{\ref*{eq:score}}
\ee
which is exactly the hypotheses space presented in sec.~\ref{sec:cac}.
Notice that if $\{f_{\theta_{d}}(\x)\}_{d\in[M]}{\subset}L^2(\R^s)$ are linearly independent (there is no reason to choose them otherwise), then so are the product functions $\{(\x_1,\ldots,\x_N)\mapsto\prod_{i=1}^N f_{\theta_{d_i}}(\x_i)\}_{d_1,\ldots,d_N\in[M]}{\subset}L^2\left((\R^s)^N\right)$ (see app.~\ref{app:hypo_space:preliminaries}), and a score function $h_y$ uniquely determines the coefficient tensor $\A^y$.  
In other words, two score functions $h_{y,1}$ and $h_{y,2}$ are identical if and only if their coefficient tensors $\A^{y,1}$ and $\A^{y,2}$ are the same.  

\subsection{Finite Function Bases for Classification of Natural Data} \label{app:hypo_space:finite_bases}

In app.~\ref{app:hypo_space:construction} we laid out the framework of classifying instances in the space $\X:=\left\{(\x_1,\ldots,\x_N):\x_i\in\R^s\right\}=(\R^s)^N$ into labels $\Y:=\{1,\ldots,Y\}$ via maximization of per-label score functions $h_y:\X\to\R$: 
$$\hat{y}(\x_1,\ldots,\x_N)=\argmax_{y\in\Y}h_y(\x_1,\ldots,\x_N)$$
where $h_y(\x_1,\ldots,\x_N)$ is of the form:
\be
h_{y}\left(\x_1,\ldots,\x_N\right)=\sum_{d_1,\ldots,d_N=1}^M A_{d_1,\ldots,d_N}^{y}\prod_{i=1}^{N} f_{\theta_{d_i}}(\x_i)
\tag{\ref*{eq:score}}
\ee
and $\{f_{\theta}\}_{d\in[M]}$ are selected from a parametric family of functions $\F=\left\{ f_{\theta}:\R^{s}\to\R\right\}_{\theta\in\Theta}$.  
For universality, i.e.~for the ability of score functions $h_y$ to approximate any function in $L^2(\X)$ as $M\to\infty$, we required that it be possible to choose a countable subset of $\F$ that is total in $L^2(\R^s)$.  
We noted that the families of Gaussians (eq.~\ref{eq:rep_gaussians}) and neurons (eq.~\ref{eq:rep_neurons}) meet this requirement.

In this subsection we formalize our argument that a \emph{finite} value for $M$ is sufficient when $\X$ represents natural data, and in particular, natural images.  
Based on empirical studies characterizing the statistical properties of natural images, and in compliance with the number of channels in a typical convolutional network layer, we find that $M$ on the order of~100 typically suffices.

\medskip

Let $\D$ be a distribution of labeled instances $(X,\bar{y})$ over $\X\times\Y$ (we use bar notation to distinguish the label $\bar{y}$ from the running index $y$), and $\D_\X$ be the induced marginal distribution of instances $X$ over $\X$.  
We would like to show, given particular assumptions on $\D$, that there exist functions $f_{\theta_1},\ldots,f_{\theta_M}\in\F$ and tensors $\A^1,\ldots,\A^Y$ of order $N$ and dimension $M$ in each mode, such that the score functions $h_y$ defined in eq.~\ref{eq:score} achieve low classification error:
\be
L_\D^{0-1} (h_1,\ldots,h_Y) := \EE_{(X,\bar{y})\sim\D} \left[ \indc{\bar{y}\neq\argmax_{y\in\Y}h_y(X)} \right]
\label{eq:class_error}
\ee
$\indc{\cdot}$ here stands for the indicator function, taking the value $1$ when its argument is true, and $0$ otherwise.

Let $\{h_y^*\}_{y\in\Y}$ be a set of ``ground truth'' score functions for which optimal prediction is achieved, or more specifically, for which the expected hinge-loss (upper bounds the 0-1 loss) is minimal:
$$ (h_1^*,\ldots,h_Y^*) = \argmin_{h'_1,\ldots,h'_Y:\X\to\R} L_\D^{hinge} (h'_1,\ldots,h'_Y) $$
where:
\be
L_\D^{hinge} (h'_1,\ldots,h'_Y):=\EE_{(X,\bar{y})\sim\D}\left[ \max_{y\in\Y} \left\{\indc{y\neq\bar{y}}+h'_y(X)\right\}-h'_{\bar{y}}(X) \right]
\label{eq:expected_hinge_loss}
\ee
Our strategy will be to select score functions $h_y$ of the format given in eq.~\ref{eq:score}, that approximate $h_y^*$ in the sense of low expected maximal absolute difference:
\be
\E:=\EE_{X\sim\D_\X}\left[\max_{y\in\Y}\abs{h_y(X)-h_y^*(X)}\right]
\label{eq:score_approx_error}
\ee
We refer to $\E$ as the \emph{score approximation error} obtained by $h_y$.  
The 0-1 loss of $h_y$ with respect to the labeled example $(X,\bar{y})\in\X\times\Y$ is bounded as follows:
\beas
&&\indc{\bar{y}\neq\argmax_{y\in\Y}h_y(X)} \leq \max_{y\in\Y} \left\{\indc{y\neq\bar{y}}+h_y(X)\right\} -h_{\bar{y}}(X) \\
&&\qquad = \max_{y\in\Y} \left\{\indc{y\neq\bar{y}}+h_y^*(X)+h_y(X)-h_y^*(X)\right\}-h_{\bar{y}}^*(X)+h_{\bar{y}}^*(X)-h_{\bar{y}}(X) \\
&&\qquad \leq \max_{y\in\Y} \left\{\indc{y\neq\bar{y}}+h_y^*(X)\right\}-h_{\bar{y}}^*(X)+\max_{y\in\Y} \left\{h_y(X)-h_y^*(X)\right\}+h_{\bar{y}}^*(X)-h_{\bar{y}}(X) \\
&&\qquad \leq \max_{y\in\Y} \left\{\indc{y\neq\bar{y}}+h_y^*(X)\right\}-h_{\bar{y}}^*(X)+2\max_{y\in\Y} \left\{\abs{h_y(X)-h_y^*(X)}\right\}
\eeas
Taking expectation of the first and last terms above with respect to $(X,\bar{y})\sim\D$, and recalling the definitions given in eq.~\ref{eq:class_error},~\ref{eq:expected_hinge_loss} and~\ref{eq:score_approx_error}, we get:
$$ L_\D^{0-1} (h_1,\ldots,h_Y) \leq L_\D^{hinge} (h_1^*,\ldots,h_Y^*)+2\E $$
In words, the classification error of the score functions $h_y$ is bounded by the optimal expected hinge-loss plus a term equal to twice their score approximation error.  
Recall that we did not constrain the optimal score functions $h_y^*$ in any way.  
Thus, assuming a label is deterministic given an instance, the optimal expected hinge-loss is essentially zero, and the classification error of $h_y$ is dominated by their score approximation error $\E$ (eq.~\ref{eq:score_approx_error}).  
Our problem thus translates to showing that $h_y$ can be selected such that $\E$ is small.

At this point we introduce our main assumption on the distribution $\D$, or more specifically, on the marginal distribution of instances $\D_\X$.  
According to various studies, in natural settings, the marginal distribution of individual vectors in $\X$, e.g.~of small patches in images, may be relatively well captured by a Gaussian Mixture Model (\emph{GMM}) with a moderate number (on the order of~100 or less) of distinct components.  
For example, it was shown in~\cite{Zoran:2012wu} that natural image patches of size $2{\times}2$, $4{\times}4$, $8{\times}8$ or~$16 {\times}16$, can essentially be modeled by GMMs with $64$ components (adding more components barely improved the log-likelihood).  
This complies with the common belief that a moderate number of low-level templates suffices in order to model the vast majority of local image patches.  
Following this line, we model the marginal distribution of $\x_i$ with a GMM having $M$ components with means $\mubf_1{\ldots}\mubf_M\in\R^s$.  
We assume that the components are well localized, i.e. that their standard deviations are small compared to the distances between means, and also compared to the variation of the target functions $h^*_y$.  
In the context of images for example, the latter two assumptions imply that a local patch can be unambiguously assigned to a template, and that the assignment of patches to templates determines the class of an image.  
Returning to general instances $X$, their probability mass will be concentrated in distinct regions of the space $\X$, in which for every $i\in[N]$, the vector $\x_i$ lies near $\mubf_{c_i}$ for some $c_i\in[M]$.  
The score functions $h^*_y$ are approximately constant in each such region.  
It is important to stress here that we do \emph{not} assume statistical independence of $\x_i$'s, only that their possible values can be quantized into $M$ templates $\mubf_1,\ldots,\mubf_M$.

Under our idealized assumptions on $\D_\X$, the expectation in the score approximation error $\E$ can be discretized as follows:
\be
\E:=\EE_{X\sim\D_\X}\left[\max_{y\in\Y} \abs{h_y(X)-h_y^*(X)} \right]= 
\sum_{c_1,\ldots,c_N=1}^M P_{c_1,\ldots,c_N} \max_{y\in\Y}\abs{h_y(\M_{c_1,\ldots,c_N})-h_y^*(\M_{c_1,\ldots,c_N})}
\label{eq:score_approx_error_discrete}
\ee
where $\M_{c_1,\ldots,c_N}:=(\mubf_{c_1},\ldots,\mubf_{c_N})$ and $P_{c_1,\ldots,c_N}$ stands for the probability that $\x_i$ lies near $\mubf_{c_i}$ for every $i\in[N]$ ($P_{c_1,\ldots,c_N}{\geq}0,~\sum_{c_1,\ldots,c_N}P_{c_1,\ldots,c_N}=1$).  

We now turn to show that $f_{\theta_1}{\ldots}f_{\theta_M}$ can be chosen to separate GMM components, i.e. such that for every $c,d\in[M]$, $f_{\theta_d}(\mubf_c)\neq0$ if and only if $c=d$.  
If the functions $f_\theta$ are Gaussians (eq.~\ref{eq:rep_gaussians}), we can simply set the mean of $f_{\theta_d}$ to $\mubf_d$, and its standard deviations to be low enough such that the function effectively vanishes at $\mubf_c$ when $c \neq d$.
If $f_\theta$ are neurons (eq.~\ref{eq:rep_neurons}), an additional requirement is needed, namely that the GMM component means $\mubf_1{\ldots}\mubf_M$ be linearly separable.  
In other words, we require that for every $d\in[M]$, there exist $\w_d\in\R^s$ and $b_d\in\R$ for which $\w_d^\top\mubf_c+b_d$ is positive if $c=d$ and negative otherwise.  
This may seem like a strict assumption at first glance, but notice that the dimension $s$ is often as large, or even larger, then the number of components $M$.  
In addition, if input vectors $\x_i$ are normalized to unit length (a standard practice with image patches for example), $\mubf_1{\ldots}\mubf_M$ will also be normalized, and thus linear separability is trivially met.  
Assuming we have linear separability, one may set $\theta_d=(\w_d,b_d)$, and for threshold or ReLU activations we indeed get $f_{\theta_d}(\mubf_c)\neq0{\iff}c=d$.
With sigmoid activations, we may need to scale $(\w_d,b_d)$ so that $\w_d^\top\mubf_c+b_d\ll0$ when $c{\neq}d$, and that would ensure that in this case $f_{\theta_d}(\mubf_c)$ effectively vanishes.

Assuming we have chosen $f_{\theta_1}{\ldots}f_{\theta_M}$ to separate GMM components, and plugging-in the format of $h_y$ given in eq.~\ref{eq:score}, we get the following convenient form for $h_y(\M_{c_1,\ldots,c_N})$:
$$ h_y(\M_{c_1,\ldots,c_N}) = A_{c_1,\ldots,c_N}^{y}\prod_{i=1}^{N} f_{\theta_{c_i}}(\mubf_{c_i}) $$
Assigning the coefficient tensors through the following rule:
$$ A^y_{c_1,\ldots,c_N}=\frac{h_y^*(\M_{c_1,\ldots,c_N})}{\prod_{i=1}^{N} f_{\theta_{c_i}}(\mubf_{c_i})} $$
implies:
$$ h_y(\M_{c_1,\ldots,c_N}) = h_y^*(\M_{c_1,\ldots,c_N}) $$
for every $y\in\Y$ and $c_1{\ldots}c_N\in[M]$.  
Plugging this into eq.~\ref{eq:score_approx_error_discrete}, we get a score approximation error of zero.

To recap, we have shown that when the parametric functions $f_\theta$ are Gaussians (eq.~\ref{eq:rep_gaussians}) or neurons (eq.~\ref{eq:rep_neurons}), not only are the score functions $h_y$ given in eq.~\ref{eq:score} universal when $M\to\infty$ (see app.~\ref{app:hypo_space:construction}), but they can also achieve zero classification error (eq.~\ref{eq:class_error}) with a moderate value of $M$ (on the order of~100) if the underlying data distribution $\D$ is ``natural''.  
In this context, $\D$ is regarded as natural if it satisfies two conditions.  
The first, which is rather mild, requires that a label be completely determined by the instance.  
For example, an image will belong to one category with probability one, and to the rest of the categories with probability zero.  
The second condition, which is far more restrictive, states that input vectors composing an instance can be quantized into a moderate number ($M$) of templates.  
The assumption that natural images exhibit this property is based on various empirical studies where it is shown to hold approximately.  
Since it does not hold exactly, our analysis is approximate, and its implication in practice is that the classification error introduced by constraining score functions to have the format given in eq.~\ref{eq:score}, is negligible compared to other sources of error (factorization of the coefficient tensors, finiteness of training data and difficulty in optimization).

\section{Related Work}\label{app:related_work}

The classic approach for theoretically analyzing the power of depth focused on investigation of the computational complexity of \emph{boolean circuits}.  
An early result, known as the ``exponential efficiency of depth'', may be summarized as follows: for every integer $k$, there are boolean functions that can be computed by a circuit comprising alternating layers of AND and OR gates which has depth $k$ and polynomial size, yet if one limits the depth to $k-1$ or less, an exponentially large circuit is required.  
See~\cite{Sipser83} for a formal statement of this classic result. 
Recently, \cite{rossman2015average}~have established a somewhat stronger result, showing cases where not only are polynomially wide shallow boolean circuits incapable of exact realization, but also of approximation (i.e. of agreeing with the target function on more than a specified fraction of input combinations).
Other classical results are related to \emph{threshold circuits}, a class of models more similar to contemporary neural networks than boolean circuits.
Namely, they can be viewed as neural networks where each neuron computes a weighted sum of its inputs (possibly including bias), followed by threshold activation ($\sigma(z)=\1[z\geq0]$).
For threshold circuits, the main known result in our context is the existence of functions that separate depth 3 from depth 2 (see~\cite{hajnal1987threshold} for a statement relating to exact realization, and the techniques in~\cite{maass1994comparison,martens2013representational} for extension to approximation).

More recent studies focus on \emph{arithmetic circuits}~(\cite{shpilka2010arithmetic}), whose nodes typically compute either a weighted sum or a product of their inputs
\footnote{
There are different definitions for arithmetic circuits in the literature.  
We adopt the definition given in~\cite{martens2014expressive}, under which an arithmetic circuit is a directed acyclic graph, where nodes
with no incoming edges correspond to inputs, nodes with no outgoing edges correspond to outputs, and the remaining nodes are either labeled as ``sum'' or as ``product''.  
A product node computes the product of its child nodes.  
A sum node computes a weighted sum of its child nodes, where the weights are parameters linked to its incoming edges.
}
(besides their role in studying expressiveness, deep networks of this class have been shown to support provably optimal training~\cite{livni2013algorithm}).  
A~special case of this are the Sum-Product Networks (\emph{SPNs}) presented in~\cite{Poon-Domingos2011}.  
SPNs are a class of deep generative models designed to efficiently compute probability density functions.  
Their summation weights are typically constrained to be non-negative (such an arithmetic circuit is called \emph{monotone}), and in addition, in order for them to be valid (i.e. to be able to compute probability density functions), additional architectural constraints are needed (e.g. decomposability and completeness).  
The most widely known theoretical arguments regarding the efficiency of depth in SPNs were given in~\cite{bengio2011shallow}.  
In this work, two specific families of SPNs were considered, both comprising alternating sum and product layers~--~a family $\F$ whose nodes form a full binary tree, and a family $\G$ with $n$ nodes per layer (excluding the output), each connected to $n-1$ nodes in the preceding layer.  
The authors show that functions implemented by these networks require an exponential number of nodes in order to be realized by shallow (single hidden-layer networks).  
The limitations of this work are twofold.  
First, as the authors note themselves, it only analyzes the ability of shallow networks to realize \emph{exactly} functions generated by deep 
networks, and does not provide any result relating to approximation.  
Second, the specific SPN families considered in this work are not universal hypothesis classes and do not resemble networks used in practice.  
Recently, \cite{martens2014expressive}~proved that there exist functions which can be efficiently computed by decomposable and complete (D\&C) SPNs of depth $d+1$, yet require a D\&C SPN of depth $d$ or less to have super-polynomial size for exact realization. 
This analysis only treats approximation in the limited case of separating depth 4 from depth 3 (D\&C) SPNs.  
Additionally, it only deals with specific separating functions, and does not convey information regarding how frequent these are.  
In other words, according to this analysis, it may be that almost all functions generated by deep networks \emph{can} be efficiently realized by shallow networks, and there are only few pathological functions for which this does not hold.  
A further limitation of this analysis is that for general $d$, the separation between depths $d+1$ and $d$ is based on a multilinear circuit result by~\cite{raz2009lower}, that translates into a network that once again does not follow the common practices of deep learning.

There have been recent attempts to analyze the efficiency of network depth in other settings as well. 
The most commonly used type of neural networks these days includes neurons that compute a weighted sum of their inputs (with bias) followed by Rectified Linear Unit (ReLU) activation ($\sigma(z)=\max\{0,z\}$).
\cite{pascanu2013number}~and~\cite{montufar2014number} study the number of linear regions that may be expressed by such networks as a function of their depth and width, thereby showing existence of functions separating deep from shallow (depth 2) networks. 
\cite{telgarsky2015representation}~shows a simple construction of a depth $d$ width 2 ReLU network that operates on one-dimensional inputs, realizing a function that cannot be approximated by ReLU networks of depth $o(d/\log{d})$ and width polynomial in $d$.
\cite{eldan2015power}~provides functions expressible by ReLU networks of depth 3 and polynomial width, which can only be approximated by a depth 2 network if the latter's width is exponential.
The result in this paper applies not only to ReLU activation, but also to the standard sigmoid ($\sigma(z)=1/(1+e^{-z})$), and more generally, to any universal activation (see assumption 1 in~\cite{eldan2015power}).
\cite{bianchini2014complexity}~also considers different types of activations, studying the topological complexity (through Betti numbers) of decision regions as a function of network depth, width and activation type.
The results in this paper establish the existence of deep vs. shallow separating functions only for the case of polynomial activation. 
While the above works do address more conventional neural networks, they do not account for the structure of convolutional networks~--~the most successful deep learning architectures to date, and more importantly, they too prove only existence of \emph{some} separating functions, without providing any insight as to how frequent these are.

We are not the first to incorporate ideas from the field of tensor analysis into deep learning.  
\cite{Socher:2013tz},~\cite{Yu:2012wy},~\cite{Setiawan:2015vn},~and~\cite{Hutchinson:2013bt} all proposed different neural network architectures that include tensor-based elements, and exhibit various advantages in terms of expressiveness and/or ease of training.  
In~\cite{Janzamin:2015uz}, an alternative algorithm for training neural networks is proposed, based on tensor decomposition and Fourier analysis, with proven generalization bounds.
In~\cite{Novikov:2014wr},~\cite{Anandkumar:2014uc},~\cite{Yang:2015wo}  and~\cite{LeSong:2013vm}, algorithms for tensor decompositions are used to estimate parameters of different graphical models.  
Notably,~\cite{LeSong:2013vm} uses the relatively new Hierarchical Tucker decomposition (\cite{Hackbusch:2009jj}) that we employ in our work, with certain similarities in the formulations.  
The works differ considerably in their objectives though: while~\cite{LeSong:2013vm} focuses on the proposal of a new training algorithm, our purpose in this work is to analyze the expressive efficiency of networks and how that depends on depth.  Recently,~\cite{Lebedev:2014vb} modeled the filters in a convolutional network as four dimensional tensors, and used the CP decomposition to construct an efficient and accurate approximation.  
Another work that draws a connection between tensor analysis and deep learning is the recent study presented in~\cite{Haeffele:2015vz}.  This work shows that with sufficiently large neural networks, no matter how training is initialized, there exists a local optimum that is accessible with gradient descent, and this local optimum is approximately equivalent to the global optimum in terms of objective value.

\section{Computation in Log-Space with SimNets} \label{app:simnets}

A practical issue one faces when implementing arithmetic circuits is the numerical instability of the product operation~--~a product node with a large number of inputs is easily susceptible to numerical overflow or underflow.  
A common solution to this is to perform the computations in log-space, i.e. instead of computing activations we compute their $\log$.  
This requires the activations to be non-negative to begin with, and alters the sum and product operations as follows.  
A product simply turns into a sum, as $\log\prod_i\alpha_i = \sum_i\log\alpha_i$.  
A sum becomes what is known as \emph{log-sum-exp} or \emph{softmax}: $\log\sum_i\alpha_i = \log\sum_i\exp(\log\alpha_i)$.  

Turning to our networks, the requirement that all activations be non-negative does not limit their universality. 
The reason for this is that the functions $f_\theta$ are non-negative in both cases of interest~--~Gaussians (eq.~\ref{eq:rep_gaussians}) and neurons (eq.~\ref{eq:rep_neurons}).
In addition, one can always add a common offset to all coefficient tensors $\A^y$, ensuring they are positive without affecting classification.
Non-negative decompositions (i.e. decompositions with all weights holding non-negative values) can then be found, leading all network activations to be non-negative.
In general, non-negative tensor decompositions may be less efficient than unconstrained decompositions, as there are cases where a non-negative tensor supports an unconstrained decomposition that is smaller than its minimal non-negative decomposition.
Nevertheless, as we shall soon see, these non-negative decompositions translate into a proven architecture, which was demonstrated to achieve comparable performance to state of the art convolutional networks, thus in practice the deterioration in efficiency does not seem to be significant.

Na\"{\i}vely implementing CP or HT model (fig.~\ref{fig:cp_model} or~\ref{fig:ht_model} respectively) in log-space translates to $\log$ activation following the locally connected linear transformations (convolutions if coefficients are shared, see sec.~\ref{sec:shared}), to product pooling turning into sum pooling, and to $\exp$ activation following the pooling. 
However, applying $\exp$ and $\log$ activations as just described, without proper handling of the inputs to each computational layer, would not result in a numerically stable computation
\footnote{
Na\"{\i}ve implementation of softmax is not numerically stable, as it involves storing $\alpha_i = \exp(\log\alpha_i)$ directly.  
This however can be easily corrected by defining $c := \max_i\log\alpha_i$, and computing $\log\sum_i\exp(\log\alpha_i-c)+c$.  
The result is identical, but now we only exponentiate negative numbers (no overflow), with at least one of these numbers equal to zero (no underflow).
}.

The SimNet architecture~(\cite{simnets1,simnets2}) naturally brings forth a numerically stable implementation of our networks.  
The architecture is based on two ingredients~--~a flexible similarity measure and the \emph{MEX} operator:
$$\textrm{MEX}_\beta(\x,\bb) := \frac{1}{\beta} \log\left(\frac{1}{N}\sum_j \exp(\beta (x_j + b_j)) \right)$$
The similarity layer, capable of computing both the common convolutional operator as well as weighted $l_p$~norm, may realize the representation by computing $\log f_\theta(\x_i)$, whereas MEX can naturally implement both log-sum-exp and sum-pooling ($\lim_{\beta\to 0}\textrm{MEX}_\beta(\x,\mathbf{0}) = \textrm{mean}_j\{x_j\}$) in a numerically stable manner.

Not only are SimNets capable of correctly and efficiently implementing our networks, but they have already been demonstrated~(\cite{simnets2}) to perform as well as state of the art convolutional networks on several image recognition benchmarks, and outperform them when computational resources are limited. 

\end{document}